\documentclass{article}
\usepackage{log_2024}						

\usepackage[T1]{fontenc}

\pdfoutput=1

\usepackage[british]{babel}

\usepackage{amsmath}
\usepackage{amssymb}
\usepackage{amsthm}
\usepackage{stmaryrd}

\usepackage[%
  giveninits   = true,
  maxcitenames = 2,
  maxbibnames  = 99,
  hyperref     = true,
  backend      = biber,
  style        = numeric,
  sorting      = nty,
  sortcites,
]{biblatex}

\usepackage{algorithm}
\usepackage{algorithmicx}
\usepackage{algpseudocode}

\makeatletter
  \algrenewcommand\ALG@beginalgorithmic{\ttfamily}
\makeatother

\usepackage{booktabs}
\usepackage{csquotes}
\usepackage{dsfont}
\usepackage{mathtools}
\usepackage{microtype}
\usepackage{mleftright}
\usepackage{multirow}
\usepackage{nicefrac}
\usepackage{paralist}
\usepackage{siunitx}
\usepackage{soul}
\usepackage{subcaption}
\usepackage{thmtools}
\usepackage{wrapfig}
\usepackage{xpatch}
\usepackage{xspace}
\usepackage{adjustbox}
\usepackage{titletoc}
\usepackage{graphicx}

\setdefaultleftmargin{1em}{1em}{1em}{1em}{1em}{1em}

\usepackage{hyperref}
\usepackage[capitalize, nameinlink]{cleveref}

\usepackage[hang, flushmargin]{footmisc}

\definecolor{bleu}     {RGB}{ 49,140,231}
\definecolor{cardinal} {RGB}{196, 30, 58}
\definecolor{celtic}   {RGB}{ 36,107,206}
\definecolor{lightgrey}{RGB}{230,230,230}

\usepackage{tikz}
\usetikzlibrary{cd}
\usepackage{pgfplots}
\pgfplotsset{compat = 1.18}

\usetikzlibrary{calc}
\usetikzlibrary{positioning}

\DefineBibliographyExtras{british}{}

\newtheorem{definition} {Definition}
\newtheorem{lemma}      {Lemma}
\newtheorem{corollary}  {Corollary}

\newtheorem{theorem}    {Theorem}

\newcommand{\ball}[2]                   {\ensuremath{\mathrm{B}_{#1}\mleft(#2\mright)}}
\newcommand{\betti}[1]                  {\ensuremath{\beta_{#1}\xspace}}
\newcommand{\boundary}               [1]{\ensuremath{\partial_{#1}}}

\newcommand{\chaingroup}             [1]{\ensuremath{C_{#1}}}
\newcommand{\diagram}                   {\ensuremath{\mathcal{D}}\xspace}

\newcommand{\edges}                     {\ensuremath{E}\xspace}
\newcommand{\graph}                     {\ensuremath{G}\xspace}
\newcommand{\graphs}                    {\mathcal{G}\xspace}
\newcommand{\homologygroup}[1]          {\ensuremath{\mathrm{H}_{#1}}}
\newcommand{\iso}                       {\ensuremath{\varphi}\xspace}
\newcommand{\multiplicity}[1]           {\ensuremath{\mu^{(#1)}}\xspace}
\newcommand{\naturals}                  {\ensuremath{\mathds{N}}\xspace}
\newcommand{\neighbours}                {\ensuremath{\mathcal{N}}\xspace}
\newcommand{\persistentbetti}        [2]{\ensuremath{\betti{#1}^{#2}}}
\newcommand{\persistenthomologygroup}[2]{\ensuremath{\homologygroup{#1}^{#2}}}
\newcommand{\reals}                     {\ensuremath{\mathds{R}}\xspace}
\newcommand{\vertices}                  {\ensuremath{V}\xspace}
\newcommand{\WL}                        {\mbox{$1$-WL}\xspace}
\newcommand{\kWL}[1][k]                 {\mbox{$#1$-FWL}\xspace}
\newcommand{\simplicialcomplex}         {\mathrm{K}}
\newcommand{\simplicialcomplexes}       {\mathcal{K}}
\newcommand{\cycle}                     {\ensuremath{\gamma}\xspace}

\newcommand*\concat{\mathbin{\|}}

\DeclareMathOperator{\degree}     {deg}
\DeclareMathOperator{\diam}       {diam}
\DeclareMathOperator{\db}         {d_B}
\DeclareMathOperator{\wasserstein}{W\!}
\DeclareMathOperator{\id}         {id}
\DeclareMathOperator{\im}         {im}

\DeclareMathOperator{\rank}       {rank}

\newcommand{\tup}     [1]{\mathbf{#1}}
\newcommand{\multiset}[1]{\left\{\!\!\left\{#1\right\}\!\!\right\}}

\tikzset{%
  vertex/.style = {%
    circle,
    fill         = black,
    draw         = black,
    minimum size = 4pt,
    inner sep    = 0pt,
  },
  edge/.style = {%
    draw,
  },
}

\title[On the Expressivity of Persistent Homology in Graph Learning]{On the Expressivity of Persistent Homology in Graph Learning}

\author[R.\ Ballester and B.\ Rieck]{%
  Rub\'en Ballester\\
  Departament de Matem\`atiques i Inform\`atica\\
  Universitat de Barcelona\\
\email{ruben.ballester@ub.edu}\And
  Bastian Rieck\\
  Department of Computer Science\\
  University of Fribourg\\
\email{bastian.grossenbacher@unifr.ch}
}

\begin{document}

\maketitle

\begin{abstract}
  Persistent homology, a technique from computational topology, has
  recently shown strong empirical performance in the context of graph
  classification. Being able to capture long range graph properties
  via higher-order topological
  features, such as cycles of arbitrary length, in combination with
  multi-scale topological descriptors, has improved predictive
  performance for data sets with prominent topological structures, such
  as molecules. At the same time, the \emph{theoretical properties} of
  persistent homology have not been formally assessed in this context.
  This paper intends to bridge the gap between computational topology and
  graph machine learning by providing a brief introduction to
  persistent homology in the context of graphs, as well as
  a theoretical discussion and empirical analysis of its expressivity
  for graph learning tasks.
\end{abstract}

\section{Introduction}

Graph learning is a highly-active research domain in machine learning,
fuelled in large parts by the \emph{geometric deep learning}~\autocite{Bronstein21a, Bronstein17a}
paradigm as well as the resurgence of new neural network architectures for handling
graph data. Methods from computational topology, by contrast, have not
yet been applied in this domain at large scales. Even though a large
amount of prior work employs topological features to solve graph learning
tasks~\autocite{Carriere20a, Chen21a, Hofer17, Hofer19, Hofer20,
Horn22a, Rieck19b, Yan22a, Ye23a, Zhao19, Zhao20a}, a formal
investigation relating expressivity in graph learning and topological
machine learning is still lacking.\footnote{
  A notable exception is recent work by \textcite{Immonen23a}, which
  analyses the expressivity of topological methods for graph
  classification tasks.
}
We believe that this is largely a deficit in communication between the
communities. This paper provides an introduction to topological
methods for graph learning, while also showing new theoretical and
empirical results about the \emph{expressivity} of topological graph
learning methods. Here, we understand expressivity as a general
concept to signify which graph properties can be captured by a method.
This includes being aware of certain substructures in
graphs~\autocite{Chen20a}, for instance, but also being able to
distinguish large classes of non-isomorphic
graphs~\autocite{Bouritsas23a, Joshi23a, Sato20a}.
While graph neural networks have demonstrated substantial gains in
this area, our paper focuses on topology-based algorithms, aiming to
provide a better understanding of their theoretical and empirical
properties.

\paragraph{Contributions.}
%
Our main theoretical contribution is a full
characterisation of the expressivity of \emph{persistent homology} in
terms of the Weisfeiler--Leman hierarchy~\autocite{Xu19a, Morris19a}.
We prove that persistent homology is \emph{at least as expressive} as
a corresponding Weisfeiler--Leman test for graph isomorphism.
Moreover, we show that there exist graphs that cannot be distinguished
using $\kWL$, the \emph{folklore Weisfeiler--Leman
algorithm}~\autocite{Morris21a}, for a specific dimension~$k$ but that can be
distinguished by persistent
homology~(with or without access to $k$-cliques in the graph). Along
the way, we also prove new properties of filtrations, i.e.\ descriptor
functions that are commonly employed to obtain topological
representations of graphs, hinting at their ability to capture
information about graph substructures.
We complement our theoretical expressivity discussions by an
experimental suite that highlights the capabilities of different
filtrations for
\begin{inparaenum}[(i)]
  \item distinguishing certain types of graphs,
  \item predicting characteristic graph properties, and
  \item serving as a baseline for classification tasks.\footnote{Our code 
    is available at \url{https://github.com/aidos-lab/PH_expressivity}.}
\end{inparaenum}

\paragraph{Guide for readers.}
%
\cref{sec:Background} briefly summarises the main concepts in graph
learning.
\cref{sec:Topological Features of Graphs} and
\cref{sec:Topological Features of Simplicial Complexes},
may be skipped by
readers that are already well versed in computational topology.
\cref{sec:Properties of Filtrations} outlines advantageous
properties of filtrations in the context of graph learning, while
\cref{sec:WL} discusses the expressivity of
persistent homology with respect to the Weisfeiler--Leman hierarchy of
graph isomorphism tests.
\cref{sec:Experiments} concludes the paper with an experimental
suite that highlights the performance of topological methods in
a variety of graph-learning tasks.

\section{Background \& Notation}\label{sec:Background}
%
\begin{figure}[tbp]
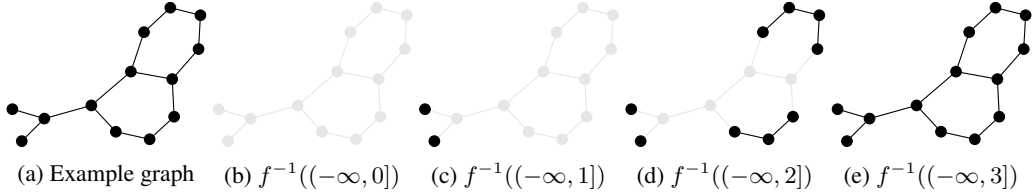

  \centering
  \subcaptionbox{Example graph}{%
    \begin{tikzpicture}
      \input{Figures/Example_graph_skeleton}

      \draw[edge] (A) -- (M); \draw[edge] (B) -- (M);
      \draw[edge] (M) -- (L); \draw[edge] (L) -- (I);
      \draw[edge] (I) -- (C); \draw[edge] (C) -- (D);
      \draw[edge] (D) -- (J); \draw[edge] (J) -- (K);
      \draw[edge] (K) -- (L); \draw[edge] (K) -- (E);
      \draw[edge] (E) -- (F); \draw[edge] (F) -- (G);
      \draw[edge] (G) -- (H); \draw[edge] (H) -- (J);

      \foreach \c in {A, B, ..., M}
      \node[vertex] at (\c) {};
    \end{tikzpicture}%
  }
  \subcaptionbox{$f^{-1}\mleft((-\infty, 0]\mright)$}{%
    \begin{tikzpicture}
      \input{Figures/Example_graph_skeleton}
    \end{tikzpicture}%
  }
  \subcaptionbox{$f^{-1}\mleft((-\infty, 1]\mright)$}{%
    \begin{tikzpicture}
      \input{Figures/Example_graph_skeleton}

      \foreach \c in {A, B}
      \node[vertex] at (\c) {};
    \end{tikzpicture}%
  }
  \subcaptionbox{$f^{-1}\mleft((-\infty, 2]\mright)$}{%
    \begin{tikzpicture}
      \input{Figures/Example_graph_skeleton}

      \foreach \c in {A, B, C, D, E, F, G, H, I}
      \node[vertex] at (\c) {};

      \draw[edge] (I) -- (C); \draw[edge] (C) -- (D);
      \draw[edge] (E) -- (F); \draw[edge] (F) -- (G);
      \draw[edge] (G) -- (H);
    \end{tikzpicture}%
  }
  \subcaptionbox{$f^{-1}\mleft((-\infty, 3]\mright)$}{%
    \begin{tikzpicture}
      \input{Figures/Example_graph_skeleton}

      \foreach \c in {A, B, ..., M}
      \node[vertex] at (\c) {};

      \draw[edge] (A) -- (M); \draw[edge] (B) -- (M);
      \draw[edge] (M) -- (L); \draw[edge] (L) -- (I);
      \draw[edge] (I) -- (C); \draw[edge] (C) -- (D);
      \draw[edge] (D) -- (J); \draw[edge] (J) -- (K);
      \draw[edge] (K) -- (L); \draw[edge] (K) -- (E);
      \draw[edge] (E) -- (F); \draw[edge] (F) -- (G);
      \draw[edge] (G) -- (H); \draw[edge] (H) -- (J);
    \end{tikzpicture}%
  }
  \caption{%
    An example graph and three different steps of a degree-based
    filtration. The respective caption indicates the pre-image of the
    corresponding filtration function.
  }
  \label{fig:Degree filtration example}
\end{figure}
%
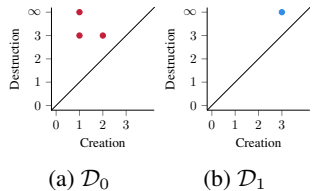
\begin{wrapfigure}[14]{l}{0.40\linewidth}
  \centering
  \subcaptionbox{$\diagram_0$}{%
    \begin{tikzpicture}[scale = 0.50]
      \begin{axis}[%
          axis x line*      = bottom,
          axis y line*      = left,
          unit vector ratio = 1 1 1,
          xmin              = 0,
          xmax              = 4,
          ymin              = 0,
          ymax              = 4,
          mark size         = 2pt,
          width             = 5cm,
          enlargelimits     = 0.05,
          tick align        = outside,
          xlabel            = {Creation},
          ylabel            = {Destruction},
          xtick             = {0, 1, 2, 3},
          xticklabels       = {$0$, $1$, $2$, $3$},
          ytick             = {0, 1, 2, 3, 4},
          yticklabels       = {$0$, $1$, $2$, $3$, $\infty$},
        ]
        \addplot[domain={-2:16}, black, no marks] {x};
        \addplot[only marks, cardinal] coordinates {
            (1, 3)
            (2, 3)
          };
        \addplot[only marks, cardinal] coordinates {
            (1, 4)
          };
      \end{axis}
    \end{tikzpicture}
  }%
  \subcaptionbox{$\diagram_1$}{%
    \begin{tikzpicture}[scale = 0.50]
      \begin{axis}[%
          axis x line*      = bottom,
          axis y line*      = left,
          unit vector ratio = 1 1 1,
          xmin              = 0,
          xmax              = 4,
          ymin              = 0,
          ymax              = 4,
          mark size         = 2pt,
          width             = 5cm,
          enlargelimits     = 0.05,
          tick align        = outside,
          xlabel            = {Creation},
          ylabel            = {Destruction},
          xtick             = {0, 1, 2, 3},
          xticklabels       = {$0$, $1$, $2$, $3$},
          ytick             = {0, 1, 2, 3, 4},
          yticklabels       = {$0$, $1$, $2$, $3$, $\infty$},
        ]
        \addplot[domain={-2:16}, black, no marks] {x};

        \addplot[only marks, bleu] coordinates {
            (3, 4)
          };
      \end{axis}
    \end{tikzpicture}
  }
  \caption{%
    \hyphenpenalty=10000
    Persistence diagrams of the filtration depicted in
    \cref{fig:Degree filtration example}.
    Points can have different multiplicities; we show \emph{essential
      features}, i.e.\ topological features that persist over the full
    filtration using an $\infty$ symbol.
  }
  \label{fig:Degree filtration example persistence diagrams}
\end{wrapfigure}

We deal with undirected graphs in this paper. An
undirected graph $\graph$ is a pair $\graph = (\vertices, \edges)$
of finite sets of~$n$ \emph{vertices} and~$m$ \emph{edges}, with $\edges
  \subseteq \mleft\{ \mleft\{u, v\mright\} \mid u, v \in \vertices, u \neq v \mright\}$.
We will also refer to an edge using tuple
notation, with the understanding that $(u, v)$ and $(v, u)$ refer to
the same edge.
Furthermore, we denote the set of all such graphs by~$\graphs$.
Two graphs $\graph = (\vertices, \edges)$ and $\graph' = (\vertices',
\edges')$ are \emph{isomorphic}
, denoted by
$\graph \simeq \graph'$, if there is a bijective function
$\iso\colon\vertices\to\vertices'$ that preserves adjacency, i.e.\ $(u,
  v) \in \edges$ if and only if $\mleft(\iso(u), \iso(v)\mright) \in \edges'$.
Since~$\iso$ is bijective, it has an inverse function, which we will
denote by~$\iso^{-1}$.
The problem of figuring out whether two graphs are isomorphic or not
is referred to as the \emph{graph isomorphism problem}. Presently,
there is no known algorithm that solves this problem in polynomial
time---efficient algorithms exist only for special families of
graphs~\autocite{Kelly57a, Colbourn81a}.
Hence, all subsequently-discussed graph isomorphism tests are perforce
limited with respect to their expressivity, i.e.\ there exist classes
of non-isomorphic graphs that they cannot distinguish. In the context
of graph isomorphism tests, we will often require the definition of
a \emph{multiset}, which is a set whose elements are included with
multiplicities. We will denote such a multiset by $\multiset{}$.

In this paper, we use the term \emph{graph expressivity} to denote two
different concepts: first, the ability of a method to distinguish
non-isomorphic graphs, and second, the ability of a method to capture
certain graph properties.
\begin{definition}[Expressivity as distinguishing non-isomorphic graphs] Let 
  $f_1\colon\graphs \to \mathcal{Y}_1$ and $f_2\colon\graphs \to \mathcal{Y}_2$ be two functions
  such that for two isomorphic graphs $\graph, \graph'$ we have that
  $f_1(\graph) = f_1(\graph')$ and $f_2(\graph) = f_2(\graph')$. 
  We say that $f_1$ is \emph{at least as expressive}
  as $f_2$ if for any pair of non-isomorphic graphs $\graph, \graph'$ we
  have that
  $f_2(\graph) \neq f_2(\graph')$ implies $f_1(\graph) \neq f_1(\graph')$.
\end{definition}
\begin{definition}[Expressivity as detecting graph properties]
  Let 
  $f_1, f_2, t\colon\graphs\to\reals^e$ be functions for some $e$ .
  We say that $f_1$ is \emph{at least as expressive} as $f_2$ according
  to property $t$ if for any graph $\graph$ we have that $f_2(\graph)
  = t$ implies $f_1(\graph) = t$. 
\end{definition}
In both cases, if $f_1$ is at least as expressive as $f_2$ and $f_2$ is
at least as expressive as $f_1$, we say that $f_1$ and $f_2$ are equally
expressive. A more complete characterization of expressivity can be
found in~\textcite[Definition~5.5]{GNNBook-ch5-li}.

\paragraph{Topological Features of Graphs}\label{sec:Topological Features of Graphs}

The simplest kind of topological features to prescribe to graphs are
\emph{connected components} and \emph{cycles}. 
Formally, we refer to the number of connected components and 
the number of~(independent) cycles in a graph as the first two 
\emph{Betti numbers} of a graph, denoted as $\betti{0}$ and $\betti{1}$,
respectively.
While the expressive power of these two numbers is limited, it can be
improved by evaluating them alongside a \emph{filtration}, i.e.\
a sequence of nested subgraphs of the form
$\emptyset \subseteq \graph_0 \subseteq \graph_1 \ldots \subseteq
  \graph_{k-1} \subseteq \graph_k = G$.
Filtrations typically arise from scalar-valued functions of the form
$f\colon\graph\to\reals$, which assign vertices and edges a value. 
Concretely, a filtration sequence of indices requires sorting the
values of the image of the filtration function~$f$, with subgraphs
$\graph_a$ being assigned the \emph{pre-images} of intervals $(-\infty,
a]$ under~$f$, i.e.\ $G_a = f^{-1}((-\infty, a])$. 
Filtrations are usually obtained from more general maps
$F\colon\graphs\to \reals^\graph, \graph \mapsto f_{\graph}$, which
assign to each graph a filtration function. In this case, we say that $F$ is
 \emph{equivariant} if, given any two graphs~$\graph, \graph'$ and
any graph isomorphism $\iso\colon\graph\to\graph'$, we have that 
$f_{\graph} = f_{\graph'}\circ \iso$.
Changes in the Betti numbers can be tracked over the course of the
filtration.
This leads to \emph{persistent Betti numbers}, which are typically
summarised in a \emph{persistence diagram}, i.e.\ a topological
descriptor, consisting of tuples $(a_i, a_j) \in \reals^2$, with $a_i$
referring to the value at which a topological feature was
\emph{created}, and $a_j$ referring to the value at which
a topological feature was destroyed~(for instance, because two connected
components are being merged).
The absolute difference in function values $|a_j - a_i|$ is called the
\emph{persistence} of a topological feature; it indicates the
prominence or relevance of said feature.
\cref{fig:Degree filtration example} depicts an example filtration for
a simple graph, where we use the degree of each vertex to filter the
graph.
The more complex structure of persistence
diagrams~(in comparison to the simple Betti numbers) already hints at
their capabilities in providing expressive graph descriptors.

\paragraph{Topological Features of Simplicial Complexes}
\label{sec:Topological Features of Simplicial Complexes}
%
Filtrations and persistence diagrams
generalise to higher dimensions as well, with the understanding that
the connectivity of higher-order structures of a graph---cliques---is
being modelled. The resulting framework is referred to as
\emph{persistent homology}.
There are different constructions for obtaining simplicial complexes
from graphs; we will subsequently use \emph{clique
complexes}~(in which a \mbox{$k$-clique} is represented by
a \mbox{$(k-1)$-simplex}) because it is
\begin{inparaenum}[(i)]
  \item straightforward to implement, and
  \item cliques are known to be characteristic graph structures~\autocite{Bouritsas23a}.
\end{inparaenum}
Persistent homology is typically calculated using matrix 
reduction algorithms, whose optimisation is still an ongoing topic of
research~\autocite{Bauer22a}.
We refer readers to \textcite{Otter17a} for a comprehensive
introduction of computational strategies. 
Readers are invited to read \cref{sec:A Primer in Computational
Topology} for a more detailed exposition of concepts in computational
topology.

\section{Properties of Filtrations}\label{sec:Properties of Filtrations}

As in the graph case, it is common to obtain simplicial complex
filtrations from a general map $F\colon \simplicialcomplex\in
\simplicialcomplexes\mapsto f_\simplicialcomplex\in
\reals^\simplicialcomplexes$ where $\simplicialcomplexes$ is the set of
all finite simplicial complexes. In the scenarios where we construct
simplicial complexes from graphs, we can promote $F$ to a function
$\mathcal{F}\colon \graph\in\graphs\mapsto (\simplicialcomplex_\graph,
f_{\simplicialcomplex_\graph})\in \simplicialcomplexes\times
\reals^\simplicialcomplexes$ that assigns to each graph a simplicial
complex with vertex set equal to the set of vertices of the graph 
and a filtration function for this simplicial complex.
We say that $\mathcal{F}$ is \emph{equivariant} if for each isomorphism
$\iso\colon \graph\to \graph'$ we have that $\iso$ induces a simplicial
complex isomorphism between $\simplicialcomplex_\graph$ and
$\simplicialcomplex_{\graph'}$ and $f_{\simplicialcomplex_\graph}
= f_{\simplicialcomplex_{\graph'}}\circ\iso$. If
$\simplicialcomplex_\graph = \graph$ for all $\graph\in\graphs$, then we
recover the previous definition of equivariance for the graph case.
A crucial result of the previous maps $\mathcal{F}$ is that, if they are
equivariant, then the persistence diagrams of the filtrations they induce are
invariant under graph isomorphism.
\begin{restatable}{proposition}{equivariantfiltrations}
  For $\mathcal{F}$ equivariant and $\graph \simeq \graph'$, the
  persistence diagrams of $f_{\simplicialcomplex_\graph}$ and
  $f_{\simplicialcomplex_{\graph'}}$ coincide.
  \label{prop:Equivariant filtrations}
\end{restatable}
\Cref{prop:Equivariant filtrations} shows that it is impossible to
`adversarially' pick an equivariant filtration generator function
$\mathcal F$ that leads to a dissimilarity between two non-isomorphic
graphs. Dropping this condition incurs a substantial loss of expressive
power.
The proposition thus guarantees that persistent homology is
compatible with equivariant function learning, pointing
towards the utility of hybrid models that leverage different types of
structural properties of graphs.
To finish this discussion of general properties, we remark that the
calculation of persistent homology carries information about the
\emph{diameter}~(the length of the longest shortest path) and
\emph{girth}~(the length of the shortest cycle) of a graph. 
\begin{restatable}{proposition}{diameterbound}
  Given \emph{any} 
  filtration~$f$ of a graph~$\graph$ with
  a single connected component such that the values of~$f$ of the endpoints
  of edges are \emph{strictly lower} than the values of their
  corresponding edges, \cref{alg:Connected components}, used to compute $\betti{0}$,
  yields an 
  upper bound of~$\diam(\graph)$.
  \label{thm:Diameter}
\end{restatable}
\begin{restatable}{proposition}{girthbound}
  Given \emph{any}
  filtration~$f$ of a graph~$\graph$,
  \cref{alg:Connected components} yields an upper bound of the girth of 
  $\graph$.
  \label{thm:Girth}
\end{restatable}
While the theoretical upper bounds are not tight, our empirical analysis
in \cref{sec:Predicting Graph Properties} shows that persistent homology
and its algorithms captures more than `just' topological
information about a graph.
Our work thus corroborates recent results in the setting of point
clouds, where persistence diagrams permit inferring additional
properties about the input data~\autocite{Bubenik20a, Lim20a,
Turkes22a}.

\section{The Weisfeiler--Leman Hierarchy}\label{sec:WL}

Having discussed the properties of specific filtrations, we now
analyse the expressivity of persistent homology in the context of the
Weisfeiler--Leman hierarchy of graph isomorphism tests. 
To this end, we categorise our filtrations~(and their corresponding 
generators) into two distinct classes:
\begin{compactenum}
  \item \label{enum:clique_filt} Filtrations operating on clique complexes 
  derived from graphs
  \item \label{enum:complete_fil} Filtrations that consider all possible 
  cliques of the vertex set up to a specified maximum dimension.
\end{compactenum}
In the first class, we incorporate only those cliques present in the
original graph, constrained to a predetermined maximum dimension.
Conversely, in the second class, we include \emph{all} feasible cliques
up to a specified dimension, irrespective of their presence in the
original graph structure.

\WL, also known as \emph{colour refinement}, constitutes a simple
method for addressing the graph isomorphism problem. It is the
backbone of graph expressivity research; readers are referred to
\textcite{Morris21a} for a comprehensive survey of \WL and its
higher-order variants. 
It is already a known result that \emph{any} \WL colouring can be
reproduced by creating a special
filtration~\autocite[Theorem~4]{Horn22a}; we restate this result as
\cref{thm:1-WL expressivity} in the appendix. For a longer discussion on
the \WL test, we refer readers to \Cref{sec:WL and Persistent Homology}.
Since \WL is also unable to distinguish between graphs with
different triangle counts or graphs with cycle information, it was
generalised to include information about labelling \emph{tuples} of
$k$~nodes~(as opposed to only labelling a single node), leading to a
hierarchy of algorithms.
The variant we shall subsequently describe is also
known as the \emph{folklore Weisfeiler--Leman
  algorithm}~\autocite{Morris21a}. It can be shown that there are
non-isomorphic graphs that cannot be distinguished by \kWL, but that
can be distinguished by \kWL[(k+1)].\footnote{%
  There are also other variants, for instance the \emph{oblivious
    Weisfeiler--Leman algorithm}. It slightly differs in the way tuples
  are being relabelled, but a paper by \textcite{Grohe21a} shows that
  the variant is essentially as powerful as \kWL~(with a minor shift
  in indices). The reader is referred to \textcite{Morris21a} and the
  references therein for an extended discussion of these aspects.
}
A well-known family of such non-isomorphic, non-distinguishable graphs
are commonly known as CFI graphs~\autocite{Cai89a}.
Following \textcite{Morris21a}, \kWL is based on the idea of assigning
colours to \emph{subgraphs} as opposed to assigning colours to
\emph{vertices}.
To achieve this, \kWL operates on \mbox{$k$-tuples} of vertices; for
iteration~$i = 0$, two tuples $\tup{v} = (v_1, \dots, v_k)$ and $\tup{w}
= (w_1, \dots, w_k)$ are assigned the same colour if the map $v_j
\mapsto w_j$ induces a graph isomorphism between the subgraphs induced
by~$\tup{v}$ and~$\tup{w}$, respectively. For subsequent iterations with
$i > 0$, we relabel the tuples similar to \WL, i.e.\
\begin{equation}
    C^{(k)}_i(\tup{v}) := \texttt{RELABEL}\mleft(\mleft(
    C^{(k)}_{i-1}(\tup{v}),
    \multiset{C^{(k)}_i\mleft(\phi_1\mleft(\tup{v}, u\mright)\mright),
    \dotsc, C^{(k)}_i\mleft(\phi_k\mleft(\tup{v}, u\mright)\mright) \mid u \in \neighbours(v)}\mright) \mright),
  \label{eq:k-WL}
\end{equation}
where $\phi_j(\tup{v}, u) := \mleft(v_1, \dotsc, v_{j-1}, u, v_{j+1},
  \dotsc, v_k\mright)$ refers to the function that replaces the $j$th
element of the \mbox{$k$-tuple} $\tup{v}$ with~$u$. This induces
a neighbourhood relation between tuples and just as in the case of
\WL, we run the algorithm until the assigned colours of tuples
stabilise for one graph. Similarly, if the colour sequences of two
graphs differ, the graphs are non-isomorphic.
As a generalisation of previous work~\autocite{Horn22a},
we can show that
\emph{any} \kWL colouring can be reproduced with 
one equivariant filtration of type~\labelcref{enum:complete_fil},
thus proving that persistent homology is \emph{at least} as expressive
as \kWL. We achieve this by showing that zero-dimensional persistence
diagrams can \emph{recover} any isomorphism-invariant real-valued
function $f\colon\graphs\to\reals$ on the space of finite graphs by
setting, for a given graph $\graph$, a constant filtration with value
equal to the function value in $\graph$, i.e.\ $f(\graph)$.
\begin{restatable}{theorem}{kwlexpressivity}
  For $k \geq 2$, there exists an equivariant filtration generator 
  $\mathcal{F}$
  of type~\labelcref{enum:complete_fil} such that its zero-dimensional 
  persistence diagrams
  are at least as expressive as \kWL[k].
  \label{thm:k-WL expressivity}
\end{restatable}
Both \cref{thm:k-WL expressivity} and \cref{thm:1-WL expressivity} may
not be completely satisfactory because they only show the
\emph{existence} of such a filtration, but make no claims about
the theoretical expressivity of filtrations of
type~\labelcref{enum:clique_filt} filtrations.
Hence, it 
is particularly interesting to understand the expressivity
of the Vietoris--Rips filtration~\autocite{Adams22a}, which is a popular 
choice in the context of persistent homology.
Given a graph $\graph = (\vertices, \edges)$ equipped with the
shortest-path distance $d_\graph$, the Vietoris--Rips filtration
generator yields the pair $(\simplicialcomplex, f_\textbf{V})$, where
$\simplicialcomplex$ is the set of all non-empty subsets $\sigma$ of
$\vertices$ such that $\diam(\sigma) \neq \infty$ and $f_{\textbf{V}}$
is the Vietoris--Rips filtration function given by
$f_{\textbf{V}}(\sigma) = \max_{u,v\in\sigma} d_\graph(u, v)$.
Notice that the Vietoris--Rips filtration generator is equivariant.
This is because graph isomorphisms preserve distances between vertices,
making the induced simplicial morphisms preserve diameters and
filtration values, and inducing an isomorphic simplicial complex
morphism.

Vietoris--Rips persistent homology cannot be directly compared to the WL
hierarchy, as no method is at least as expressive as the other.
For example, Vietoris--Rips persistence diagrams of dimension zero can
capture the number of connected components of graphs, while the \WL
hierarchy cannot always capture them. See~\cref{lem:Connected components
WL} for a proof of this fact. By contrast, the \WL test can
capture the difference between a path graph and a cycle graph both of
length~$3$, while Vietoris--Rips persistent homology cannot.
The fact that both methods are not comparable in terms of expressivity
suggests that persistent homology methods based on Vietoris--Rips
filtrations and WL/message passing methods are complimentary, since
many standard graph neural networks are exactly as expressive as
\WL~\autocite[Theorem~2]{Morris19a}.
Regarding filtrations of type~\labelcref{enum:clique_filt},
we do not expect cliques and high-dimensional persistent homology to be 
crucial in distinguishing pairs of non-isomorphic graphs. 
On the one hand, equivariant graph filtration generators of
type~\labelcref{enum:clique_filt} are already \emph{at least as
expressive} as \WL, thus they can distinguish almost all non-isomorphic
graph pairs~\autocite[Theorem~3.3]{Kiefer20a}.
On the other hand, to prove that persistent homology for filtrations of
type~\labelcref{enum:clique_filt} is at least as expressive as \kWL for
$k\geq 2$, we would need to distinguish non-isomorphic CFI pairs for all
$k\geq 2$. However, by definition of CFI graphs, they do not contain
cliques of size three or greater, see~\cref{lem:CFITrianglesGeneral}, 
meaning that most of the expressive
power of these filtrations can only arise from the values in the
original graph structure and low-dimensional persistent homology.
We would ideally want to extend \cref{thm:k-WL expressivity} to state
that persistent homology is \emph{strictly} more expressive than \kWL,
although this is not as straightforward as for the case $k=1$.
Currently, we can provide one such counterexample, described in
\cref{tab:Strongly-regular graphs}, consisting of the $4 \times 4$
rook's graph and the Shrikhande graph. With an appropriate filtration,
persistent homology can distinguish these two graphs
\emph{without} requiring more than vertices and edges, whereas \kWL[2]
is unable to distinguish them. We leave a general result for future work.

\section{Experiments}\label{sec:Experiments}

The previous sections discussed the theoretical properties of
filtrations. Here, we analyse their \emph{empirical} performance as a 
complement to the theoretical discussion to demonstrate the high expressivity of 
persistent homology in graph tasks. We
first analyse the \emph{expressivity} of five different well-known
filtrations by
letting them distinguish non-isomorphic graphs. This is followed by
experiments on graph property prediction. Please refer to
\cref{app:Expressivity Experiments,app:Graph Classification Experiments}
for additional expressivity and graph-classification experiments.

\paragraph{Experimental setup.}
%
We use different data sets containing strongly-regular graphs~\autocite{McKayGraphs},
minimal Cayley graphs, as well as benchmark data sets for graph-learning
tasks~\autocite[BREC]{Wang23a}.
In the following, we will use five different filtrations for each graph:
\begin{compactenum}
  \item A \emph{degree} filtration~(denoted by \textbf{D}), i.e.\ $v \mapsto \degree(v)$. The
  degree filtration is the most basic non-trivial filtration of
  a graph, showing nevertheless surprising empirical performance in
  graph classification tasks~\autocite{Hofer20, OBray21a, Rieck20c}.
  \item A filtration based on the eigenvalues of the \emph{undirected
    graph Laplacian}~(denoted by \textbf{L}), i.e.\ $v \mapsto \lambda_v$, where $\lambda_v$
  indicates the eigenvalue of the undirected graph Laplacian
  corresponding to vertex~$v$. The graph Laplacian is known to
  capture characteristic properties of a graph; in the context of
  persistent homology it is often used in the form of a \emph{heat
    kernel signature}~\autocite{Carriere20a}.
  \item A filtration based on the \emph{Ollivier--Ricci
    curvature}~\autocite{Ollivier07a}~(denoted by O) in the graph, setting $v \mapsto
    -1$ and $(u, v) \mapsto \kappa(u, v)$, with $\kappa$ denoting the
  Ollivier--Ricci curvature
  \begin{equation}
    \kappa(u, v) := 1 - \wasserstein_1\mleft(\mu_u^{\alpha}, \mu_v^{\alpha}\mright),
    \label{eq:Ollivier--Ricci curvature}
  \end{equation}
  where $\wasserstein_1$ denotes the first Wasserstein
  distance,\footnote{%
    This metric is also known as the \emph{Earth Mover's Distance}~\autocite{Levina01a}.
    The Wasserstein distance is a fundamental concept in optimal
    transport; the monograph by \textcite{Villani09} contains
    a comprehensive introduction to this topic.
  }
  and $\mu_u^{\alpha}, \mu_v^{\alpha}$ denote \emph{probability
  measures} based on a lazy random walk in the graph, i.e.\
  $\mu_u^{\alpha}(u) := \alpha$, indicating the probability of staying
  at the same vertex, $\mu_u^{\alpha}(v) :=  (1
  - \alpha)\nicefrac{1}{\degree(u)}$ for a neighbour~$v$ of~$u$, and
  $\mu_u^{\alpha}(\cdot) := 0$ otherwise.
  The probability measures in \cref{eq:Ollivier--Ricci curvature}, i.e.\
  $\mu_u^\alpha$, may be adjusted; recent work investigates the utility
  of this perspective~\autocite{Coupette23a, Southern23a}.
  We set $\alpha = 0$ for our subsequent experiments, thus obtaining
  a non-lazy random walk, and leave the investigation of the impact of other
  values for future work.
  \item A filtration based on the \emph{augmented Forman--Ricci
    curvature}~\autocite{Samal18a}~(denoted by \textbf{F}), where we
    again set $v \mapsto -1$ and $(u, v) \mapsto \mathcal{F}(u, v)$,
    with
    $\mathcal{F}(u,v) := 4-\degree(u)-\degree(v)+3\left|\neighbours(u)\cap
    \neighbours(v)\right|$.
  \item A \emph{Vietoris--Rips} filtration~(denoted by \textbf{V}) over
    the metric space defined by the shortest-path distance between its
    nodes~\autocite{Adams22a}.
\end{compactenum}
We exclude colouring filtrations, as used in~\cref{thm:k-WL
expressivity}, from our analysis due to their practical infeasibility.
For each input graph $\graph$, this approach would require generating
$k!$ filtrations and constructing simplicial complexes containing all
cliques up to dimension~$k$.
After picking a filtration~(except for the Vietoris--Rips one), we
expand the graph by filling in all \mbox{$(k+1)$-cliques}, with
filtration value for a clique $\sigma$ given recursively by the maximum
filtration value of its proper subcliques, i.e.\
$\sigma\mapsto\max_{\tau \subsetneq \sigma}f(\tau)$, and calculating
persistent homology up to dimension~$k$. Hence, for $k = 1$, we leave
the graph `as-is,' making use of connected components and cycles
only.\footnote{
  Notice the shift in dimension: a $k$-simplex has $k+1$
  vertices, meaning that persistent homology in dimension~$k$ contains
  information about \mbox{$(k+1)$}-cliques.
}
Our persistent homology calculations result in a set of persistence
diagrams for each graph, which we compare in a pairwise manner using
the bottleneck distance described by \cref{eq:Bottleneck distance}.
We consider two graphs to be different whenever the distance between
their persistence diagrams is $> \num{1e-8}$, i.e.\ above machine
precision.
This setup has the advantage that 
no additional classifier is
required
; when dealing with data sets of pairs of non-isomorphic
graphs, we may thus simple \emph{count} the number of non-zero
distance pairs, showing the utility of persistent homology as a powerful
baseline for graph expressivity analysis.
We omit comparisons between persistent homology and isolated filtration
values, as persistent homology is at least as expressive as filtration
values alone. This is a consequence of the \emph{pairing lemma}, which
ensures all filtration values, with multiplicities, are distributed as
births or deaths in persistence diagrams up to the dimension of the
simplicial complex. Hence, the filtration values can always be
reconstructed from the persistence diagrams.

\subsection{Strongly-Regular Graphs and Minimal Cayley Graphs}

\begin{table}[tbp]
  \sisetup{
    table-format    = 1.2,
    round-mode      = places,
    round-precision = 2,
    detect-all      = true,
    detect-weight   = true,
  }
  \centering
  \caption{Success rate~($\uparrow$) for distinguishing pairs of
    \emph{strongly-regular graphs} when using different
    filtrations at varying expansion levels
    of the graph~(denoted by~$k$). \kWL[2] cannot distinguish between
    any of these pairs.
  }%
  \label{tab:Strongly-regular graphs}
  \smallskip
  \begingroup
  \let\b=\bfseries
  \let\e=\itshape
  \setlength{\tabcolsep}{3pt}
  \begin{tabular}{lSSSSSSSSSSSSSSS}
    \toprule
    \multirow{5}{*}{\e Data} & \multicolumn{5}{c}{$k = 1$}           & \multicolumn{5}{c}{$k = 2$}                & \multicolumn{5}{c}{$k = 3$}                \\
    \cmidrule(lr){2-16}
                             & \multicolumn{15}{c}{\e Filtration}                                                                                              \\
    \cmidrule(lr){2-16}
                             & {D}    & {O}     & {F}  & {L}  & {V}  & {D}    & {O}    & {F}    & {L}    & {V}    & {D}    & {O}    & {F}    & {L}    & {V}    \\
    \midrule
    \texttt{16622}           & 0.00   & \b1.00  & 0.00 & 0.00 & 0.00 & \b1.00 & \b1.00 & \b1.00 & \b1.00 & \b1.00 & \b1.00 & \b1.00 & \b1.00 & \b1.00 & \b1.00 \\
    \texttt{251256}          & 0.00   &   0.00  & 0.00 & 0.00 & 0.00 &   0.00 &   0.00 &   0.00 &   0.00 &   0.00 & \b0.90 & \b0.90 & \b0.90 & \b0.90 & \b0.90 \\
    \texttt{261034}          & 0.00   &   0.00  & 0.00 & 0.00 & 0.00 &   0.20 &   0.20 &   0.20 &   0.76 &   0.20 &   0.93 &   0.93 &   0.93 & \b0.98 &   0.93 \\
    \texttt{281264}          & 0.00   &   0.83  & 0.00 & 0.00 & 0.00 &   0.00 & \b1.00 &   0.00 &   0.00 &   0.00 & \b1.00 & \b1.00 & \b1.00 & \b1.00 & \b1.00 \\
    \texttt{291467}          & 0.00   &   0.00  & 0.00 & 0.00 & 0.00 &   0.00 &   0.00 &   0.00 &   0.00 &   0.00 & \b0.77 & \b0.77 & \b0.77 & \b0.77 & \b0.77 \\
    \texttt{351668}          & 0.00   &   0.00  & 0.00 & 0.00 & 0.00 &   0.00 &   0.00 &   0.00 &   0.05 &   0.00 &   0.95 &   0.95 &   0.95 & \b0.99 &   0.95 \\
    \texttt{351899}          & 0.00   &   0.00  & 0.00 & 0.00 & 0.00 &   0.00 &   0.00 &   0.00 &   0.00 &   0.00 & \b0.81 & \b0.81 & \b0.81 & \b0.81 & \b0.81 \\
    \texttt{361446}          & 0.00   &   0.00  & 0.00 & 0.00 & 0.00 &   0.02 &   0.02 &   0.02 &   0.83 &   0.02 &   0.92 &   0.92 &   0.92 & \b0.99 &   0.92 \\
    \texttt{401224}          & 0.00   &   0.00  & 0.00 & 0.00 & 0.00 &   0.93 &   0.93 &   0.93 &   0.99 &   0.93 &   0.94 &   0.94 &   0.94 & \b0.99 &   0.94 \\
    \bottomrule
  \end{tabular}
  \endgroup
\end{table}

We start our investigation by analysing \emph{strongly-regular
graphs}, which are are known to be extremely
challenging to distinguish. \kWL[2], for instance, cannot distinguish
\emph{any} of these graphs~\autocite[Section~3.3]{Morris21a}.
\cref{tab:Strongly-regular graphs}
summarises the performance of our selected filtrations.
We first observe that for $k = 1$, i.e.\ for the original graph
without any cliques, few pairs of graphs can be distinguished
by the five filtrations.
Notably, a curvature-based filtration \emph{is} sufficient to distinguish the
two graphs in the \texttt{16622} data set, colloquially known
as the $4 \times 4$ rook's graph and the Shrikhande graph.
Distinguishing between these two graphs is usually said to require
knowledge about cliques~\autocite{Bodnar21a}, but it turns out that
  a suitable filtration is sufficient. However, the empirical
expressivity of curvature-based filtrations appears limited for $k
  = 1$, improving only for higher-order clique complexes. The Laplacian
filtration, by contrast, exhibits strong empirical performance for $k
  = 2$ on almost half of the data sets, increasing to near-perfect
performance for $k = 3$ in almost
all data sets. Vietoris--Rips and degree filtrations obtain the exact
same success rates for every~$k$ and data set, exhibiting the lowest
success rates for $k=1$ and $k=2$, and the same ones to the curvature
filtrations for $k=3$. It is clear that knowledge about higher-order
cliques helps in driving performance here. Notice that in contrast to
other algorithms~\autocite{Bodnar21a}, no additional embedding of the
graphs is required; we are comparing `raw' persistence diagrams
directly.

\begin{table}[tbp]
  \sisetup{
    table-format    = 1.2,
    round-mode      = places,
    round-precision = 2,
    detect-all      = true,
    detect-weight   = true,
  }
  \centering
  \caption{
    Success rate~($\uparrow$) for distinguishing pairs of \emph{minimal
    Cayley graphs} when using five different filtrations at varying
    expansion levels of the graph~(denoted by~$k$). Values for \WL are
    shown as a baseline.
  }
  \label{tab:Minimal Cayley graphs}
  \smallskip
  \begingroup
  \let\b=\bfseries
  \let\e=\itshape
  \setlength{\tabcolsep}{3pt}
  \begin{tabular}{lSSSSSSSSSSS}
    \toprule
    \multirow{4}{*}{\e Data} &        & \multicolumn{5}{c}{$k = 1$}     & \multicolumn{5}{c}{$k = 2$}                                                                                  \\
    \cmidrule(lr){3-12}
                          &        & \multicolumn{10}{c}{\e Filtration}                                                                                                                \\
    \cmidrule(lr){3-12}
                          & {$1$-WL} & {D}                               & {O}                           & {F}      & {L}      & {V}      & {D}      & {O}      & {F}        & {L}               & {V}      \\
    \midrule
    \texttt{cay12}        & 0.67 & 0.67                          & 0.71                      & 0.95 & 0.86 & 0.00 & 0.95 & 0.95 & 0.95 & \b1.00 & 0.95 \\
    \texttt{cay16}        & 0.83 & 0.83                          & 0.42                      & 0.83 & 0.58 & 0.00 & 0.83 & 0.92 & 0.83 & \b1.00 & 0.94 \\
    \texttt{cay20}        & 0.61 & 0.61                          & 0.46                      & 0.61 & 0.79 & 0.00 & 0.61 & 0.79 & 0.61 & \b1.00 & 0.89 \\
    \texttt{cay24}        & 0.65 & 0.65                          & 0.82                      & 0.86 & 0.98 & 0.00 & 0.83 & 0.93 & 0.86 & \b1.00 & 0.93 \\
    \texttt{cay32}        & 0.76 & 0.76                          & 0.81                      & 0.76 & 0.90 & 0.00 & 0.76 & 0.94 & 0.76 & \b1.00 & 0.90 \\
    \texttt{cay36}        & 0.69 & 0.69                          & 0.87                      & 0.84 & 0.99 & 0.00 & 0.84 & 0.95 & 0.84 & \b1.00 & 0.94 \\
    \texttt{cay60}        & 0.69 & 0.69                          & 0.90                      & 0.78 & 1.00 & 0.00 & 0.77 & 0.95 & 0.78 & \b1.00 & 0.97 \\
    \texttt{cay63}        & 0.49 & 0.49                          & 0.89                      & 0.73 & 0.88 & 0.00 & 0.73 & 0.93 & 0.73 & \b1.00 & 0.96 \\
    \bottomrule
  \end{tabular}
  \endgroup
\end{table}

As an additional class of complex graphs, we analyse \emph{minimal
  Cayley graphs}, i.e.\
Cayley graphs that encode a group with a minimal generating set. Minimal
Cayley graphs are still a topic of active research in graph theory, with
several conjectures yet to be proven~\autocite{Babai78a, Babai96a}.
Specifically, isomorphisms of Cayley graphs have been extensively 
studied~\autocite{CAIHENG1998109,LI2002301,morris2016isomorphismscayleygraphsnilpotent,Alspach1997}
and can be associated to isomorphisms of subsets of groups and interesting 
questions about their structure; see~\autocite[Section~3]{CAIHENG1998109}.
We
follow the same experimental setup as described above but also show
the performance of \WL, calculated via a subtree Weisfeiler--Leman
graph kernel~\autocite{Shervashidze11a}. 
\Cref{tab:Minimal Cayley graphs} shows the results. We observe that
the Laplacian filtration is trivially able to distinguish between all
these graphs for $k=2$ and has a strong performance for $k=1$; this is not
surprising since spectra of the graphs are known to be
characteristic~\autocite{Lovasz75a}.
The performance of the
curvature filtrations also points towards the utility of this
formulation in practice.
We also observe that Vietoris--Rips filtrations are the ones that most
benefit from the availability of higher-order information, with
a significant improvement in performance, going from a $0$\% success
rate for $k = 1$ to $> 90$\% for $k = 2$.

\begin{table}[tbp]
  \centering
  \caption{
    Success rate~($\uparrow$) for distinguishing pairs of instances of the
    \emph{BREC data set} when using different filtrations at varying
    expansion levels of the graph~(denoted $k$). Due to combinatorial
    constraints, we did not calculate the Vietoris--Rips filtration for $k=4$.
    Legend and number of graphs per category: \texttt{B}~(Basic, 60), \texttt{R}~(Regular, 100),
    \texttt{E}~(Extension, 100), \texttt{C}~(CFI, 100), \texttt{4, 20}~($4$-Vertex
    Condition), \texttt{D}~(Distance-Regular, 20) graphs, respectively
    and \texttt{A}~(average over full data set, 400 graphs).
  }
  \label{tab:BREC}
  \smallskip
  \sisetup{
    table-format    = 1.2,
    round-mode      = places,
    round-precision = 2,
    detect-all      = true,
    detect-weight   = true,
  }
  \begingroup
  \let\b=\bfseries
  \let\e=\itshape
  \setlength{\tabcolsep}{2pt}
  \resizebox{\linewidth}{!}{%
  \begin{tabular}{lSSSSSSSSSSSSSSSSSSS}
    \toprule
                        & \multicolumn{5}{c}{$k = 1$}               & \multicolumn{5}{c}{$k = 2$}               & \multicolumn{5}{c}{$k = 3$}                     & \multicolumn{4}{c}{$k = 4$}            \\
    \midrule\multirow{2.5}{*}{\e Data}
                        & \multicolumn{19}{c}{\e Filtration}            \\
    \cmidrule(lr){2-20}
                        & {D}   & {O}    & {F}   & {L}     & {V}    & {D}   & {O}     & {F}   & {L}     & {V}   & {D}     & {O}     & {F}     & {L}     & {V}     & {D}      & {O}     & {F}     & {L}     \\
    \midrule                                                                                                                                                                                                        
    \texttt{B}          & 0.033 & 0.933  & 0.867 & \b1.000 & 0.000  & 0.783 & \b1.000 & 0.983 & \b1.000 & 0.517 &   0.833 & \b1.000 & 0.983   & \b1.000 &   0.583 &   0.833  & \b1.000 &   0.983 & \b1.000 \\
    \texttt{R}          & 0.000 & 0.420  & 0.320 &   0.000 & 0.000  & 0.390 &   0.540 & 0.500 &   0.480 & 0.390 &   0.850 &   0.930 & 0.910   &   0.930 &   0.850 &   0.890  & \b0.970 &   0.950 & \b0.970 \\
    \texttt{E}          & 0.070 & 0.760  & 0.440 &   0.940 & 0.000  & 0.260 &   0.920 & 0.590 & \b1.000 & 0.110 &   0.290 &   0.920 & 0.590   & \b1.000 &   0.160 &   0.290  &   0.920 &   0.590 & \b1.000 \\
    \texttt{C}          & 0.030 & 0.030  & 0.030 & \b0.060 & 0.030  & 0.030 &   0.030 & 0.030 & \b0.060 & 0.030 &   0.030 &   0.030 & 0.030   & \b0.060 &   0.030 &   0.030  &   0.030 &   0.030 &   0.060 \\
    \texttt{4}          & 0.000 & 0.000  & 0.000 &   0.000 & 0.000  & 0.000 &   0.000 & 0.000 &   0.000 & 0.000 & \b1.000 & \b1.000 & \b1.000 & \b1.000 & \b1.000 & \b1.000  & \b1.000 & \b1.000 & \b1.000 \\
    \texttt{D}          & 0.000 & 0.000  & 0.000 & \b0.050 & 0.000  & 0.000 &   0.000 & 0.000 & \b0.050 & 0.000 &   0.000 &   0.000 &   0.000 & \b0.050 & \b0.050 &   0.000  &   0.000 &   0.000 &   0.050 \\
    \midrule
    \texttt{A}          & 0.030 & 0.443  & 0.328 &   0.403 & 0.008  & 0.287 &   0.522 & 0.427 &   0.537 & 0.210 &   0.468 &   0.670 &   0.580 &   0.700 &   0.400 &   0.477  &   0.680 &   0.590 & \b0.710 \\
    \bottomrule
  \end{tabular}
  }
  \endgroup
\end{table}

\subsection{BREC Data Set}\label{sec:BREC Data}

BREC~\autocite{Wang23a} is a novel graph expressivity data set focused on providing
a robust, challenging benchmark for graph isomorphism detection,
containing particularly difficult graph classes. The data set consists
of $400$ graphs, divided into $6$ different categories, \cref{tab:BREC}
provides an overview of them and shows our experimentally-observed success rates.
We observe that the Vietoris--Rips filtration is the worst performing
filtration for all tested $k$ values, followed by the degree filtration.
The most informative curvature-based filtration is the Ollivier--Ricci
one, which obtained higher or equal success rates than the Forman--Ricci
curvature filtration in all the subsets of data, with strictly higher
average success rates for all $k$ values. The most effective filtration
overall is the Laplacian filtration for $k=4$, surpassing almost all the
algorithms, including both graph neural networks and classical methods,
described in~\textcite[Table~2]{Wang23a}. The only exception was the
$N_2$ algorithm~\autocite{Paap24a} that obtained a success rate of 74.5\%,
compared to the 71\% obtained by the Laplacian filtration for $k=4$.
Given the fact that $N_2$ requires knowledge of the isomorphism class of
\emph{all} \mbox{$2$-hop}-induced subgraphs of a graph, the
computational complexity makes it infeasible to apply for many graph
sizes in practices~\autocite[Section~4.3]{Paap24a}, whereas the Laplacian filtration
remains computable.
Overall, these results experiments underscore
the high expressivity of persistent homology, making it a strong baseline for
graph-learning tasks.
Please refer to \cref{sec:Additional
Results for the BREC Data Set} for additional results.

\subsection{Predicting Graph Properties}\label{sec:Predicting Graph Properties}

\begin{table}[tbp]
  \sisetup{
    table-format    = 1.2,
    round-mode      = places,
    round-precision = 2,
    detect-all      = true,
    detect-weight   = true,
  }
  \centering
  \caption{
    Accuracy~($\uparrow$) when predicting the properties of graphs in the
    \texttt{ogbg-molhiv} molecular graph data set~\autocite{hu20a} using
    different filtrations at varying expansion levels of the graph~(denoted by~$k$). 
    Column R contains the average probability of successfully predicting
    a property at random over all possible values of
    the property in the data, with the probability of choosing a label
    being proportional to the number of graphs with that label.
  }
  \label{tab:Radii prediction experiments}
  \smallskip
  \begingroup
  \let\b=\bfseries
  \let\e=\itshape
  \setlength{\tabcolsep}{4pt}
  \begin{tabular}{lSSSSSSSSSSS}
    \toprule
    \multirow{5}{*}{\e Data} &  & \multicolumn{5}{c}{$k = 1$} & \multicolumn{5}{c}{$k = 2$} \\
    \cmidrule(lr){3-12}
     & & \multicolumn{10}{c}{\e Filtration} \\
    \cmidrule(lr){3-12}
     & {R} & {D} & {O} & {F} & {L} & {V} & {D} & {O} & {F} & {L} & {V} \\
    \midrule
    Diameter & 0.02 & 0.09 & \b0.11 & 0.05 & 0.08 & 0.07 & 0.10 & 0.08 & 0.06 & 0.09 & {-} \\
    Girth & 0.04 & 0.00 & 0.11 & 0.34 & 0.46 & \b0.48 & 0.14 & 0.21 & 0.33 & 0.45 & {-} \\
    Radius & 0.03 & 0.16 & \b0.21 & 0.06 & 0.15 & 0.14 & 0.20 & 0.18 & 0.10 & 0.17 & {-} \\
    \bottomrule
    \end{tabular}
  \endgroup
\end{table}

As our final expressivity experiments, we assess 
the capability of persistent homology to predict graph properties. Here,
we focus on the \emph{diameter}, the \emph{radius}, and the \emph{girth}.

\paragraph{Predicting the diameter of 
random graphs.}
We predict the diameter of Erdős--Rényi and Watts--Strogatz graphs.
For the Erdős--Rényi graphs, we generate $N = 100$ graphs with~$n = 100$ vertices 
and $p = 0.1$. This
edge probability corresponds to the critical connectivity
regime, for which closed-form solutions of the diameter distribution
are not readily available~\autocite{Hartmann18a}. We assess the
utility of persistent homology by specifying a regression
task;\footnote{%
  Following \textcite{Hartmann18a}, we take the diameter of an
  Erdős--Rényi graph, which might consist of different connected
  components, to be the largest diameter of all connected components
  of the graph.
} to
this end, we vectorise the persistence diagrams for each filtration
using \emph{Betti curves}~\autocite{OBray21a, Rieck20c}, a simple
curve-based topological representation.
While this
representation is technically a function, we represent it as
a histogram of~$10$ bins and train a \emph{ridge regression
  classifier} via leave-one-out cross-validation to predict the
diameter of each graph. We deliberately focus only on zero-dimensional
and one-dimensional persistent homology, i.e.\ we leave $k = 1$.
Using the \emph{mean absolute error}~(MAE) for evaluation, we find
that Ollivier--Ricci curvature performs best~(\num{0.057}), followed by
the Laplacian spectrum~(\num{0.061}), and the degree
filtration~(\num{0.065}).
We observe similar patterns when calculating $N = 100$ Watts--Strogatz graphs of
type~$(100, 5, 0.1)$, i.e.\ we keep the same number of vertices and
the same edge rewiring probability~$p$, but connect each node with
its~$5$ nearest neighbours in a ring neighbourhood. Using the same
classifier, we again find that Ollivier--Ricci curvature achieves the
lowest MAE~(\num{0.851}), followed by the degree
filtration~(\num{0.865}), and the Laplacian spectrum~(\num{1.072}).

\paragraph{Predicting graph properties of the \texttt{ogbg-molhiv}~\autocite{hu20a} data set.}
We predict the maximum radius and diameter among the radii and diameters
of the connected components of each graph as well as their girth.
See \cref{fig:ogbg-molhiv_radii} for a visualisation of the distribution
of these properties across all graphs in the data set.
We use a random forest regression model and \emph{persistence
images}~\autocite{Adams17a} as its input, computed from the persistence
diagrams of the graphs calculated as in the previous sections.
The model is trained on the usual train and test splits of the data set,
with the final prediction obtained by rounding to the nearest
integer.\footnote{%
  We predict a girth of $\infty$ in case the output value is larger than
  the number of nodes in the graph.
}
Overall, this is a challenging task, with $24$, $15$, and $5$ labels having
more than $100$ examples for the diameter, radius and girth, respectively.
Thus, the average probabilities of guessing the correct label by
random choice is $0.02$, $0.03$, and $0.04$.
However, we find that persistent homology, with a suitable filtration,
can predict the maximum radius, the maximum diameter, and the
girth of the graphs on unseen examples~(i.e.\ on the test data set)
in approximately $11$\%, $20\%$, and $48$\% of the cases, respectively,
exhibiting substantially-improved results over a random baseline.
For predicting radii and diameters, we find that the
Ollivier--Ricci curvature outperforms the other filtrations for $k=1$
but gets worse for $k=2$, where the degree filtration performs the best.
It is surprising that the Vietoris--Rips filtration is \emph{not} able to
predict the radii of the graphs in this data set accurately ($\approx
14$\% accuracy), being the only filtration that works with explicit
distances. However, this filtration proves most effective for predicting
the girth, having an accuracy of almost $50$\% as compared to the $4$\%
of the baseline. 
These results complement prior work on predicting
properties of \mbox{point clouds~\autocite{Turkes22a, Bubenik20a}},
showing that topology-based graph-learning approaches carry a large
degree of additional information about graphs
.

\section{Discussion}\label{sec:Discussion}

We discussed various aspects of the computation and provided evidence of
the advantageous properties of persistent homology in the context of
graph learning.
Our primary theoretical insight is that
persistent homology is \emph{at least as expressive} as
a corresponding \WL or \kWL test, in some cases surpassing their
discriminative power.
Experiments underscore these theoretical expressivity properties,
while also demonstrating that persistent homology is able to capture
additional properties of a graph.
In light of the performance
differences among the different filtration functions 
in our experiments, we suggest that future work should focus
on elucidating properties of classes of such functions, aiming to strike
a balance between expressivity and efficiency.
Another limitation involves the calculation \emph{per se}, which
requires costly clique-finding operations and is thus not scalable to
large, dense graphs.
If node features are present, alternative geometrical-topological
approaches could potentially be used~\autocite{Maggs24a, Roell24a,
Munch23a, Marsh23a, Turner14b}, necessitating additional research.

In the future, we would like to formally prove which classes of
filtrations make \mbox{$k$-dimensional} persistent homology strictly
more expressive than \kWL~(if any).
Follow-up research could also focus on identifying other
properties~(next to the diameter and girth) that can be captured by
persistent homology in graph learning. Such research has both
theoretical and empirical components; a first step would be
a formalisation of which substructures are captured by
models~\autocite{Bouritsas23a, Chen20a, Southern23a}.
Moreover, the success of the Laplacian filtration at the experimental
tasks may hint at new filtrations based on spectral graph theory that
provide a trade-off between utility and computational efficiency.
We find that this research direction is overlooked by the
computational topology research community, with most of the
expressivity/stability results focusing on describing the stability of
distance-based filtrations under perturbations~\autocite{Chazal14a,
  Chazal17a}, and few works focusing on graphs~\autocite{Bauer21a}.
Based on our experiments, we thus envision that persistent homology will
constitute a strong baseline for graph-learning applications.
As previous work shows, even topology-inspired approaches, making use
of concepts such as filtrations, can approximate the performance of
highly-parametrised models at a fraction of the computational
cost~\autocite{OBray21a}.
All insights obtained using such topological methods hint at the
overall utility of graph-structural information for graph learning
tasks, but it is not clear whether current graph benchmark data
actually exhibit such structures~\autocite{Palowitch22a}.
We thus hope that persistent homology and related techniques will also
find more applications in \emph{hybrid models}, which are able to
incorporate geometrical--topological information about graphs.
This is an emerging research topic of crucial relevance since there are now
numerous graph data sets that combine geometrical information~(node
coordinates) with topological information~\autocite{Joshi23a}.

Our theoretical analysis of the properties of persistent homology for
graph learning tasks show the potential and benefits of a topology-based
perspective. We are confident that additional computational topology
concepts will enrich and augment machine learning models, leading to
new insights about their theoretical and empirical capabilities.
This paper is but a first attempt at elucidating the theoretical
utility of computational topology in a graph learning context;
advancing the field will require many more insights.

\clearpage

\section*{Acknowledgements}

This paper was motivated in large parts by discussions with participants of
the BIRS 2022 Workshop on `Deep Exploration of Non-Euclidean Data with
Geometric and Topological Representation Learning.'
The authors are indebted to Dr.\ Leslie O'Bray for comments and
discussions that helped substantially improve the main arguments of 
the paper.
Moreover, the authors are also grateful for the stimulating discussions
with the anonymous reviewers, in particular reviewer \texttt{9rsi},  and
the area chair, who believed in the merits of this work.
Rubén Ballester was supported by the Ministry of Science, Innovation and Universities through  
projects PID2019-105093GB-I00, PID2020-117971GB-C22, and 
PID2022-136436NB-I00 and through the 
FPU contract FPU21/00968, and by the Departament de Recerca i Universitats de 
la Generalitat de Catalunya (2021 SGR 00697).
Bastian Rieck was partially supported by the Bavarian state government
with funds from the \emph{Hightech Agenda Bavaria}. This work has
received funding from the Swiss State Secretariat for Education,
Research, and Innovation~(SERI). The funders had no role in the
preparation of the manuscript or the decision to publish.

\printbibliography
\clearpage

\appendix

\startcontents
\printcontents{}{1}{{%
    \vskip10pt\hrule
    \large\textbf{Appendix~(Supplementary Materials)}\vskip3pt\hrule\vskip5pt}
}
\clearpage

\counterwithin*{figure}{part}
\stepcounter{part}
\renewcommand{\thefigure}{S.\arabic{figure}}

\counterwithin*{table}{part}
\stepcounter{part}
\renewcommand{\thetable}{S.\arabic{table}}

\section{Extended literature review}
  Topological Data Analysis has been used in a variety of applications in 
  machine learning, such as in the development of topological input features, 
  the analysis of learning algorithms, or the development of new,
  topology-aware models; see~\cite{surveymltda} for a general introduction 
  of TDA in machine learning and~\cite{ballester2024topologicaldataanalysisneural} 
  for a more specific review of TDA in neural networks.

  In the context of graph learning, the use of topological data analysis, 
  and particularly persistent homology has contributed many new insights and tools.
  In the context of architectures, persistent homology has been used as pooling 
  layers~\cite{topopoolinggraphs,ying2024boosting}, readout layers~\cite{pmlr-v198-zhang22b},
  and regular layers~\cite{Horn22a,pmlr-v235-verma24a,Ye23a}. In the context of expressivity and 
  graph neural network property analysis,~\cite{Horn22a} also proved that persistent homology 
  is at least as expressive as \WL and~\cite{Immonen23a} extended this result and our results 
  for persistent homology for vertex- and edge-based filtrations.

\section{Counting Connected Components}\label{sec:Counting Connected Components}
%
Since the main text deals with higher-order topological features, and
such features afford a substantially less intuitive grasp, we want to
briefly comment on how to obtain~$\betti{0}$, the number of connected
components. As with many problems in computer science, this procedure
turns out to be simple if we pick our data structures correctly. Here,
we need a \emph{union--find} data structure, also known as a disjoint
set forest.
This data structure is built on the vertices of a graph and affords
two operations, viz.\ \texttt{union}~(or \texttt{merge}) and \texttt{find}. The
\texttt{merge} operation assigns two vertices to the same connected
component, while the \texttt{find} operation returns the current
connected component of a vertex. Building such a data structure is
reasonably easy in programming languages like \texttt{Python}, which
offer \emph{associative arrays}.
\cref{alg:Connected components} shows one particular pseudo-code
implementation of a simple union--find data structure. The pseudo-code
assumes that all operations are changing objects `in place.' Notice
that the \texttt{find} operation is implemented implicitly via
a lookup in the \texttt{merge} function. A proper object-oriented
implementation of a union--find data structure should have these two
operations in its public interface.

\begin{algorithm}
  \caption{%
    Using associative arrays to find connected components
  }
  \label{alg:Connected components}
  \begin{algorithmic}[1]
    \Function{get\_connected\_components}{$\vertices, \edges$}
    \State $\texttt{UF} \gets \{\}$
    \For{$v \in \vertices$}
    \State $\texttt{UF}[v] \gets v$
    \EndFor
    \For{$e = (v, w)\in \edges$}
    \State \Call{merge}{$\texttt{UF}, v, w$}
    \EndFor
    \State\Return $\{ v \mid \texttt{UF}[v] = v\}$
    \EndFunction
    \medskip
    \Function{merge}{$\texttt{UF}, v, w$}
    \If{$\texttt{UF}[v] \neq \texttt{UF}[w]$}
    \State $\texttt{UF}[v] \gets w$
    \EndIf
    \EndFunction
  \end{algorithmic}
\end{algorithm}

\section{A Primer in Computational Topology}\label{sec:A Primer in Computational Topology}

With the understanding that our readers have different backgrounds,
the following section provides a primer of the most relevant concepts
in computational topology.

\subsection{Simplicial Homology}\label{sec:Simplicial Homology}

The Betti numbers of a graph are actually a special instance of a more
generic concept, known as \emph{simplicial homology}. We will see that
under this framework, the Betti numbers are the ranks of the zeroth
and first homology group, respectively.
Simplicial homology is not required in order to understand most of the
results of this paper, but an appreciation for some of the concepts
will be helpful in understanding connections to other concepts. We try
to provide a self-contained introduction to the most relevant concepts
and refer to the textbook by \textcite{Munkres84} for a more in-depth
exposition of these concepts. We start by introducing the central
object of algebraic topology---the \emph{simplicial complex}.\footnote{
  Technically, we will be working with \emph{abstract simplicial
    complexes}. A definition of a simplicial complex in terms of convex
  subsets is also possible, but this necessitates understanding
  certain nuances that are irrelevant to this paper.
}
\begin{definition}[Simplicial complex]
  A \emph{simplicial complex}~$\simplicialcomplex$ is a system of sets that
  is closed under the subset operation. Thus, for any $\sigma \in
    \simplicialcomplex$ and $\tau \subseteq \sigma$, we have $\tau \in
    \simplicialcomplex$.
  An element~$\sigma \in \simplicialcomplex$ with $|\sigma| = k + 1$
  is also referred to as a \mbox{$k$-simplex}. We also express this by
  writing $\dim \sigma = k$. Moreover, if $k$ is
  maximal among all the simplices of $\simplicialcomplex$, we say that
  the $\simplicialcomplex$ is a \mbox{$k$-dimensional} simplicial
  complex.
\end{definition}
Note that there is an unfortunate shift in dimensions: a \mbox{$k$-simplex}
has indeed $k+1$ elements. This convention makes sense when we relate
it to the concept of \emph{dimension}. A \mbox{$0$-simplex}, i.e.\ a point
or a vertex, should be assigned a dimension of~$0$. The reader should
thus mentally equate the dimension of a simplex with its dimension.
The text will aim to quell any confusion about such shifts.
The quintessential example of a simplicial complex is a graph~$\graph
  = (\vertices, \edges)$. Setting $\simplicialcomplex := \vertices \cup
  \edges$, we obtain a \mbox{$1$-dimensional} simplicial complex. We may
calculate additional types of simplicial complexes from a graph, for
instance by \emph{expanding} each \mbox{$(k+1)$}-clique into
a \mbox{$k$-simplex}~\autocite{Horak09a, Rieck18a}.

The simplicial complex on its own is only a set system; to perform
calculations with this type of data structure, we need to imbue it
with additional operations. One of the most common operations involves
defining homomorphisms between the subsets of a simplicial
complex~$\simplicialcomplex$.
\begin{definition}[Chain group of a simplicial complex]
  Given a simplicial complex~$\simplicialcomplex$, the
  vector space generated over $\mathds{Z}_2$ coefficients whose
  elements are the $k$-simplices of~$\simplicialcomplex$ is called the
  \emph{$k$th chain group}, denoted by $\chaingroup{k}(\simplicialcomplex)$.
  The elements of a chain group are also
  referred to as \emph{simplicial chains}.
\end{definition}
Elements of the chain group are thus sums of simplices of
a compatible dimension. For instance, we may write the sum of all
edges of a graph to obtain a valid simplicial chain. Operating over
$\mathds{Z}_2$ coefficients means that $\sigma + \sigma = 0$, the
empty chain, for all $\sigma \in \simplicialcomplex$.\footnote{%
  Readers familiar with algebraic topology will recognise
  $\mathds{Z}_2$ as a deliberate choice of \emph{coefficient field}
  for the subsequent calculations. Other choices are possible, but the
  computational topology community predominantly uses $\mathds{Z}_2$
  coefficients in practice, with very few
  exceptions~\autocite{Gardner22a}.
  However, all the proofs and concepts introduced in this paper apply,
  \emph{mutatis mutandis}, for other coefficient sets as well.
}
Simplicial chains permit us to define homomorphisms
between chain groups, which will ultimately permit us to treat
topological questions with tools of linear algebra.
\begin{definition}[Boundary homomorphism]
  Given $\sigma = (v_0,\dots,v_k) \in \simplicialcomplex$, we define
  the \emph{$k$th boundary homomorphism}
  $\boundary{k}\colon\chaingroup{k}(\simplicialcomplex)\to\chaingroup{k-1}(\simplicialcomplex)$
  as
  \begin{equation}
    \boundary{k}(\sigma) := \sum_{i=0}^{k}(v_0,\dots, v_{i-1},v_{i+1},\dots, v_k),
    \label{eq:Degree filtration example}
  \end{equation}
  i.e.\ a sum of simplices with the $i$th entry---vertex---of the
  simplex missing, respectively.
\end{definition}

\begin{figure}
  \centering
  \begin{minipage}{0.25\textwidth}
    \begin{tikzpicture}[dot/.style={circle, fill, inner sep=1pt}]
      \node [label=below:$a$,dot] (a) at (0,0) {};
      \node [label=below:$b$,dot] (b) at (2,1) {};
      \node [label=above:$c$,dot] (c) at (1,2) {};

      \draw (a) -- (b);
      \draw (a) -- (c);
      \draw (b) -- (c);
    \end{tikzpicture}
  \end{minipage}
  \begin{minipage}{0.70\textwidth}
    \small
    The triangle is a simple simplicial complex, consisting of one
    $2$-simplex, three $1$-simplices and three $0$-simplices,
    respectively.
    The boundary of the $2$-simplex is non-zero: we have $\boundary{2}\{a,b,c\} = \{b,c\} + \{a,c\}
      + \{a,b\}$.
    The set of edges, on the other hand, does not have a boundary, i.e.\ $\boundary{1}\left(\{b,c\}
      + \{a,c\} + \{a,b\}\right) = \{c\} + \{b\} + \{c\} + \{a\} + \{b\} + \{a\} = 0$, because the
    simplices cancel each other out.
  \end{minipage}
  \caption[An example boundary homomorphism calculation]{%
    Calculating the boundaries of a $2$-simplex and the boundary of
    a simplicial chain consisting of $1$-simplices.
    Notice that the boundary of a boundary is always zero. This is
    a fundamental property of persistent homology.
    The figure is slightly adapted from \textcite{Rieck17d}.
  }
  \label{fig:Boundary operator example}
\end{figure}
%
It is sufficient to define $\boundary{k}$ on individual
simplices; since it is a homomorphism, it extends to arbitrary
simplicial chains. \cref{fig:Boundary operator example} shows an
example of this calculation.
The boundary operator already assigns some algebraic structure
to~$\simplicialcomplex$, but it turns out that we can use it to assign
a set of groups to the simplicial complex.
\begin{definition}[Homology group]
  We define the \emph{$k$th homology group} of a simplicial
  complex~$\simplicialcomplex$ as
  \begin{equation}
    \homologygroup{k}(\simplicialcomplex) := \ker\boundary{k} / \im\boundary{k+1},
  \end{equation}
  i.e.\ a \emph{quotient} group that we obtain from the two subgroups
  $\ker\boundary{k}$ and $\im\boundary{k+1}$ of $\chaingroup{k}$.
\end{definition}
The $k$th homology group of a simplicial complex contains its
\mbox{$k$-dimensional} topological features in the form of
\emph{equivalence classes} of simplicial chains, also known as
\emph{homology classes}.
This rather abstract definition is best understood by an additional
simplification step that involves calculating the \emph{rank} of
a homology group.
\begin{definition}[Betti number]
  The rank of the $k$th homology
  group~$\homologygroup{k}(\simplicialcomplex)$ is known as the
  \emph{$k$th Betti number}, denoted by $\betti{k}$.
\end{definition}
Despite the rank being a rather coarse summary, Betti numbers turn out
to be of immense utility in comparing different simplicial complexes.
We may even reproduce 
the equation for the cyclomatic number, $\betti{1} = m + \betti{0} - n$,
by noting that the
\emph{Euler characteristic}~$\chi(\simplicialcomplex) := \sum_i
  (-1)^i\ \mleft|\mleft\{ \sigma \mid \dim \sigma = i \mright\}\mright|$ can
also be expressed as a sum of alternating Betti numbers, i.e.\
$\chi(\simplicialcomplex) := \sum_i (-1)^i \betti{i}$. For a proof of
this surprising fact, see e.g.\ \textcite[p.\ 124]{Munkres84}.
Using this equivalence, we see that we can calculate $\betti{1}$ by
reshuffling some of the terms, thus also explaining why
the equation exhibits alternating signs.

At this point, we have introduced a large amount of algebraic
machinery. Changing our focus back to graphs, we may reap some
advantageous properties by noting that homology groups are somewhat
preserved under graph isomorphism.\footnote{%
  The reader well-versed in algebraic topology may be aware of this
  property directly, but we find it useful to mention this fact
  briefly.
}

\begin{lemma}
  \label{lem:Graph homology isomorphism}
  Let $\graph, \graph'$ be two isomorphic graphs.
  Then the homology groups of $\graph$ and $\graph'$ are isomorphic,
  i.e.\ $\homologygroup{k}(\graph) \simeq \homologygroup{k}(\graph')$
  for all~$k\geq 0$.
  \label{lem:Homology isomorphism}
\end{lemma}
\Cref{lem:Graph homology isomorphism} is a direct consequence of 
the \emph{functoriality} of homology. Functoriality of homology implies that,
given any map $f\colon\graph\to\graph'$, 
there is an induced map $\homologygroup{k}(f)\colon\homologygroup{p}(\graph)\to\homologygroup{p}(\graph')$ 
between homology groups for all $p\geq 0$,
and that, given two maps $f\colon\graph\to\graph'$ and $g\colon\graph'\to\graph''$,
we have $\homologygroup{k}(g\circ f) = \homologygroup{k}(g)\circ \homologygroup{k}(f)$. 
In particular, the identity map
$\id\colon\graph\to\graph$ induces the identity map on the homology groups of $G$, making 
that an isomorphism between $\graph$ and $\graph'$ induces an isomorphism 
between their homology groups. Thus, the proof of the lemma is straightforward.

\begin{proof}
  Let $k\geq 0$ and let $\iso\colon\graph\to\graph'$ be an isomorphism between 
  $\graph$ and $\graph'$. By \emph{functoriality} of homology, 
  $\homologygroup{k}(\iso)$ is an isomorphism between
  $\homologygroup{k}(\graph)$ and $\homologygroup{k}(\graph')$.  
\end{proof}

As a direct corollary, the Betti numbers of~$\graph$ and~$\graph'$
do not change, and in fact, a similar property holds for isomorphic
simplicial complexes.

\begin{corollary}
  The Betti numbers of isomorphic graphs are equal, i.e.\
  $\betti{p}(\graph) = \betti{p}(\graph')$ for all~$p$.
  \label{cor:Betti isomorphism}
\end{corollary}

This may be seen as a hint about the popularity of simplicial homology
in algebraic topology: the framework leads directly to characteristic
descriptions that remain invariant under~(graph) isomorphism.

\subsection{Persistent Homology}

Because of their conceptual simplicity---their calculation in low
dimensions only involves knowledge about the connected components of
a graph---Betti numbers are somewhat limited in their expressivity.
Taking any graph~$\graph = (\vertices, \edges)$, even the addition
of a single edge to~$\graph$ will change its Betti numbers, either by
merging two connected components~(thus decreasing $\betti{0}$) or by
creating an additional cycle~(thus increasing $\betti{1}$).
This is a direct consequence of the definition of Euler's formula, 
which effectively states that the insertion of a new edge~$e = (u, v)$ with
$u, v \in \vertices$ either causes~$\betti{1}$ to increase by~$1$
because $m$ changes, or remain the same in case the number of
connected components~$\betti{0}$ changes. However, a single
edge may only merge two connected components into one, so
$\betti{0}$ may also at most decrease by~$1$. This indicates that
Betti numbers are too coarse to be practically useful in large-scale
graph analysis. It is possible to turn Betti numbers into
a \emph{multi-scale descriptor} of a graph. This requires certain
modifications to the previously-introduced concepts. Similar to
\cref{sec:Simplicial Homology}, we will formulate everything in terms
of simplicial complexes, again pointing out that this results in
a more general description.

\begin{definition}[Filtration]
  Given a simplicial complex~$\simplicialcomplex$, we call a sequence
  of simplicial complexes \emph{filtration} if it affords a nesting
  property of the form
  \begin{equation}
    \emptyset = \simplicialcomplex_0 \subseteq \simplicialcomplex_1
    \subseteq \dots \subseteq \simplicialcomplex_{m-1} \subseteq
    \simplicialcomplex_m = \simplicialcomplex.
    \label{eq:Filtration}
  \end{equation}
  Since each element of this sequence is a valid simplicial complex,
  we can also think of this construction as `growing'
  $\simplicialcomplex$ by adding simplices one after the other.
  \label{def:Filtration}
\end{definition}
Filtrations arise naturally when building simplicial complexes from
point cloud data, but even in the context of graphs, we can imagine
filtrations as \emph{filtering} a graph based on some type of data,
or function, assigned to its vertices. For instance, we may build a
filtration of a graph based on the degree of its vertices, defining
$\simplicialcomplex_i$ to be the subgraph consisting of all vertices
satisfying the degree condition, plus all edges whose endpoints
satisfy it, i.e.\
\begin{equation}
  \simplicialcomplex_i := \{v \in \vertices \mid \degree(v) \leq i\}
  \cup \mleft\{\{u, v\} \in \edges \mid \degree(u) \leq i \wedge
  \degree(v) \leq i\mright\}.
\end{equation}
Notice that we could also express the second condition more compactly
by assigning to each \mbox{$1$-simplex}~(each edge) the \emph{maximum}
of the weight of its vertices. This construction is sometimes also
referred to as a \emph{lower-star filtration} since it extends
a node-level function to higher-order simplices~\autocite{Dey22a}.
Not all filtrations have to be defined on the vertex level; as long as
each edge in the filtration is preceded by its vertices, we can also
build valid filtrations from functions that are primarily defined on
edges.\footnote{%
  In \cref{sec:Experiments}, we will make use of a filtration defined
  on edge-based curvature values.
}

Setting aside further discussions about how to obtain filtrations for
now, filtrations are compatible with the simplicial homology framework
introduced above. The boundary operators $\boundary{}(\cdot)$,
together with the inclusion homomorphism between consecutive
simplicial complexes, induce a homomorphism between corresponding homology
groups of any filtration of~$m$ simplicial complexes.
Given $i \leq j$, we write $\iota^{i,j} \colon \homologygroup{k}(\simplicialcomplex_i) \to \homologygroup{k}(\simplicialcomplex_j)$ to denote this homomorphism.
This construction yields a sequence of homology groups
\begin{equation}
  0 = \homologygroup{k}(\simplicialcomplex_0) \xrightarrow{\iota_k^{0,1}} \homologygroup{k}(\simplicialcomplex_1) \xrightarrow{\iota_k^{1,2}}  \dots \xrightarrow{\iota_k^{m-2,m-1}} \homologygroup{k}(\simplicialcomplex_{m-1}) \xrightarrow{\iota_k^{m-1,m}} \homologygroup{k}(\simplicialcomplex_m) = \homologygroup{k}(\simplicialcomplex)
\end{equation}
for every dimension~$k$. We then define the \emph{$k$th persistent
  homology group} as
\begin{equation}
  \persistenthomologygroup{d}{i,j} :=
  \ker\boundary{k}(\simplicialcomplex_i) / \mleft(
  \im\boundary{k+1}\mleft(\simplicialcomplex_j\mright)\cap\ker\boundary{k}\mleft(\simplicialcomplex_i\mright)\mright),
\end{equation}
containing all topological features---homology classes---created in
$\simplicialcomplex_i$ that still exist in $\simplicialcomplex_j$.
Following the definition of the ordinary Betti number, we then define
the  \emph{$k$th persistent Betti number} to be the rank of this
group, leading to $\persistentbetti{k}{i,j} := \rank \persistenthomologygroup{k}{i,j}$.
It should be noted that this type of construction makes use of
numerous deep mathematical concepts; for the sake of an expository
article, we heavily summarise and compress everything to the most
pertinent results.

The appeal of persistent homology can be seen when we start to make
use of the features it captures. If we assume that our filtration is
associated with a set of values~$a_0 \leq a_1 \leq \dots \leq a_{m-1}
  \leq a_m$, such as the function values on the vertices, we can
calculate \emph{persistence diagrams}, i.e.\ simple topological
feature descriptors.
\begin{definition}[Persistence diagram]
  The \emph{$k$-dimensional persistence diagram} of a filtration is
  the multiset of points in $\reals^2$ that, for each pair $i, j$ with
  $i \leq j$, stores the tuple $(a_i, a_j)$ with multiplicity
  \begin{equation}
    \mu_{i,j}^{(k)} := \mleft( \persistentbetti{k}{i,j-1} - \persistentbetti{k}{i,j} \mright) - \mleft( \persistentbetti{k}{i-1,j-1} - \persistentbetti{k}{i-1,j}
    \mright).
  \end{equation}
  We will also assign a multiplicity to \emph{essential topological features} of
  the simplicial complex, setting
  \begin{equation}
    \mu_{i,\infty}^{(k)} := \persistentbetti{k}{i,m} - \persistentbetti{k}{i-1,m},
  \end{equation}
  which denotes all features that are still present in the last
  simplicial complex of the filtration, i.e.\ in $\simplicialcomplex_m
    = \simplicialcomplex$. The persistence diagram thus contain all the
  information carried by Betti numbers.
\end{definition}
Persistence diagrams summarise the topological activity of
a filtration. Given a persistence diagram~$\diagram$, for any tuple
$(a_i, a_j)$, the quantity $| a_j - a_i|$ is called the
\emph{persistence} of the respective topological feature. Persistence
indicates whether a feature, created in some simplicial complex during
the filtration, is prominent or not. This notion was originally
introduced by \textcite{Edelsbrunner02} to analyse the relevance of
topological features of a distance function; the terminology is
supposed to indicate the prominence of a topological feature.
Features with a high persistence are commonly taken to be relevant, whereas
features with a low persistence used to be considered as noise; this
assumption is changing as, depending on the filtration, low
persistence may also just imply `low reliability.'~\autocite{Bendich16a}

Persistence diagrams can be endowed with different metrics and
kernels~\autocite{Kwitt15a, Reininghaus15a}, and it is known that the
space of persistence diagrams is an Alexandrov space with curvature
bounded from below~\autocite{Turner14a}. The most common metric to
compare two persistence diagrams is the \emph{bottleneck distance},
defined as
\begin{equation}
  \db\mleft(\diagram, \diagram'\mright) := \inf_{\eta\colon \diagram \to \diagram'}\sup_{x\in \diagram}\|x-\eta(x)\|_\infty,
  \label{eq:Bottleneck distance}
\end{equation}
where $\eta$ ranges over all bijections between the two persistence
diagrams. \cref{eq:Bottleneck distance} is solved using optimal
transport; different cardinalities are handled by permitting points in
one diagram to be transported to their corresponding projection on the
diagonal. Another metric is the 
\emph{Wasserstein distance}~\autocite[Theorem~7.3]{wassersteinDistance}, 
in which
the $\sup$ calculation is replaced by a weighted sum over all
distances between points in a diagram.

Stability properties of filtrations are a crucial aspect of research in
computational topology~\autocite{Skraba20a}.
Given two filtrations~$f, g$ of the same simplicial complex, a seminal
result by \textcite{Cohen-Steiner07} proves the following bound:
\begin{theorem}[Bottleneck stability]
  Let $f, g$ refer to filtrations of a simplicial
  complex~$\simplicialcomplex$, and let $\diagram_f$ and $\diagram_g$
  denote their respective persistence diagrams. The \emph{bottleneck
  distance} distance is upper-bounded by
  $
  \db\mleft(\diagram_f, \diagram_g\mright) \leq \|f-g\|_\infty
  $,
  where $\|\cdot\|_\infty$ refers to the supremum norm.
\end{theorem}
An extension of this theorem, with $\db$ being replaced by the
Wasserstein distance, shows that persistence diagrams are also
stable in the Lipschitz sense~\autocite{Cohen-Steiner10}. These
stability properties are remarkable because they link a topological
quantity with a geometrical one, thus underscoring how persistent
homology itself incorporates both geometrical and topological aspects
of input data.


\subsection{Computational complexity of persistent homology}
In this section, we briefly discuss the computational complexity of computing 
persistence diagrams. For a more detailed discussion, 
we refer the reader 
to~\autocite{Otter17a}.

Recall that persistence diagrams are obtained by first constructing a filtration
of a simplicial complex. Then, persistence diagrams are computed from the homology 
groups of the simplicial complexes and the induced maps from inclusion between them. 
Therefore, the computational complexity of computing persistence diagrams is the sum 
of the complexity of building the filtration, and the complexity of the algorithm 
constructing the persistence diagrams. 

Depending on the filtration type, the former complexity can greatly vary. For instance,
the complexity of building a filtration based on the degree of vertices is quadratic in
the number of vertices using the adjacency matrix if we do not add simplices of higher
dimension than edges. On the other hand, the complexity of building a Vietoris--Rips 
filtration is high in general, having $\mathcal{O}(n^{k+1})$ simplices up to dimension $k$ 
for a set of $n$ vertices, leading to research on more efficient filtrations similar to the 
original Vietoris--Rips one~\autocite{Sheehy2013}. For the filtration of 
\Cref{thm:k-WL expressivity}, computing 
the filtration is equivalent to compute the $k$-WL colouring of the graph, which can be 
done in $\mathcal{O}(n^{k+1}\log n)$ where $n$ is the number of 
vertices~\autocite[Section~1]{lichter2024computationalcomplexityweisfeilerlemandimension}.

For the latter, several algorithms exist 
depending on the filtration type and the homology coefficients used. A notable example is 
Ripser~\autocite{Ripser}, which computes persistence diagrams for Vietoris--Rips filtrations 
in an efficient way. Using the standard algorithm, the worst case complexity is cubic 
in the number of simplices $m$, being this a sharp bound. For dimension zero, alternatives 
based on the connection between the single linkage clustering algorithm 
and zero-dimensional persistence 
diagrams~\autocite[Claim~5.1]{single_linkage} allows computation in $\mathcal{O}(n^2)$.


\section{\WL and Persistent Homology}
\label{sec:WL and Persistent Homology}
Formally, \WL proceeds by iteratively calculating a node colouring
function~$C^{(1)}_i \colon \vertices \to \naturals$. The output of
this function depends on the neighbours of a given node. For
a vertex~$v$ at iteration~$i > 0$, we have
\begin{equation}
  C^{(1)}_i(v) := \texttt{RELABEL}\mleft(\mleft( C^{(1)}_{i-1}(v), \multiset{C^{(1)}_i(u) \mid u \in \neighbours(v)}\mright) \mright),
  \label{eq:1-WL}
\end{equation}
where \texttt{RELABEL} refers to an injective function that maps the
tuple of colours to a unique colour, i.e.\ a unique number 
and $\neighbours(v)$ refers to the neighbors of $v$, i.e.\ its 
adjacent vertices. The
algorithm is initialised by either using existing labels or the degree
of vertices.\footnote{%
  Note that \cref{eq:1-WL} does not recognise an ordering of labels.
  Initialising \WL with a constant value thus leads to the \emph{same}
  colouring---up to renaming---after the first iteration.
}
After a finite number of steps, the colour assignments generated using
\cref{eq:1-WL} stabilise. If two graphs give rise to different colour
sequences, the graphs are guaranteed to be non-isomorphic. The \WL
test is computationally easy and constitutes an
upper bound for the expressivity of many graph neural network~(GNN)
architectures~\autocite{Morris19a, Xu19a}. In other words, if \WL
cannot distinguish two non-isomorphic graphs, GNNs will also not be
able to distinguish them.
As stated in~\cref{sec:WL}, it is already a known result that 
\emph{any} \WL colouring can be reproduced by
creating a special filtration.
The implication is that persistent homology is \emph{at least as
expressive} as \WL because there is a filtration that distinguishes all
the graphs \WL can distinguish.
In fact,~\cref{lem:Connected components WL} demonstrates that the \WL
algorithm fails to detect even the most fundamental topological property
of graphs: the number of connected components. This limitation stands in
contrast to the capabilities of (persistent) homology, which can readily
capture such basic topological features.
\cref{sec:Topology and WL} shows
examples of graphs that can be distinguished based on their topological
features but not by the \WL algorithm. These examples demonstrate that
a topological perspective is strictly more expressive than \WL
.
%

\section{Theorems \& Proofs}
\label{sec:Proofs}

For the convenience of the reader, we restate all results again before
providing their proof.

When working with filtrations in the subsequent proofs, it would be
ideal to have filtrations that satisfy \emph{injectivity} on the level
of vertices, i.e.\ $f(v) \neq f(v')$ if $v \neq v'$. Such injective
filtrations have the advantage of permitting gradient-based optimisation
schemes~\cite{Hofer20}. The following lemma, first proved in
\textcite{Horn22a}, demonstrates that injectivity is not a strict
requirement, though, as it is always possible to find an injective
filtration function that is arbitrarily close~(in the Hausdorff sense)
to a non-injective filtration function.

\begin{lemma}
  For all $\epsilon > 0$ and a filtration function~$f$ defined on the
  vertices, i.e.\ $f\colon \vertices \to \reals^d$, there is an
  injective function $\tilde{f} \colon \vertices \to \reals^d$
  such that $\|f - \tilde{f}\|_\infty < \epsilon$.
  \label{lem:Injective function}
\end{lemma}
\begin{proof}
  Let $V = \{v_1, \dots, v_n\}$ be the vertices of a graph and $\im
    f = \{u_1, \dots, u_m\}$ be their images under~$f$. Since~$f$ is not
  injective, we have $m < n$. We resolve non-injective vertex pairs
  iteratively.
  For $u \in \im f$, let $V' := \{v \in V \mid f(v) = u\}$. If $V'$ only
  contains a single element, we do not have to do anything. Otherwise,
  for each $v' \in V'$, pick a new value from $\ball{\epsilon}{u}
    \setminus \im f$, where $\ball{r}{x} \subset \reals^d$ refers to the
  open ball of radius~$r$ around a point~$x$~(for $d = 1$, this
  becomes an open interval in~$\reals$, but the same reasoning applies
  in higher dimensions).
  Since we only ever
  remove a finite number of points, such a new value always exists, and
  we can modify~$\im f$ accordingly. The number of vertex pairs for
  which~$f$ is non-injective decreases by at least one in every
  iteration, hence after a finite number of iterations, we have
  modified~$f$ to obtain~$\tilde{f}$, an \emph{injective} approximation
  to~$f$. By always picking new values from balls of radius $\epsilon$,
  we ensure that $\|f - \tilde{f}\|_\infty < \epsilon$, as required.
\end{proof}

\equivariantfiltrations*
The proof of~\Cref{prop:Equivariant filtrations} is a direct consequence of the
following lemma.
\begin{lemma}
  Let $\mathds{K}$ be any field. If $\mathcal{F}$ is equivariant and
  $\graph \simeq \graph'$, then the persistence modules over
  $\mathds{K}$ obtained by applying the homology functor
  $\homologygroup{k}$ to the chain complexes generated by the sublevel
  set filtrations $f_{K_\graph}$ and $f_{K_{\graph'}}$ are isomorphic
  for any $k\geq 0$.
  \label{lem:Equivariant persistence modules}
\end{lemma}
\begin{proof}
  Let $\iso\colon\graph\to\graph'$ be an isomorphism between
  $\graph$ and $\graph'$, and let $a\in\reals$.
  We claim that $\iso$ induces
  a simplicial isomorphism between $K_{\graph}(a) = f_{K_\graph}^{-1}((-\infty, a])$ and
  $K_{\graph'}(a) = f_{K_{\graph'}}^{-1}((-\infty, a])$. First, note that 
  $\sigma\in K_{\graph}(a)$ implies $\iso(\sigma) \in K_{\graph'}(a)$.
  Then, the restriction $\iso_a$ of the map $\iso\colon K_{\graph} \to K_{\graph'}$ 
  to $K_{\graph}(a)$ is a well-defined simplicial complex morphism between $K_{\graph}(a)$ and
  $K_{\graph'}(a)$.
  In particular, $\iso_a$ is an isomorphism. The injectivity of $\iso_a$
  stems from the fact that $\iso$ is an isomorphism between simplicial
  complexes.
  The surjectivity of $\iso_a$ comes from the fact that
  $f_{K_\graph}(\iso^{-1}(\tau))=f_{K_{\graph'}}(\tau)$ for all $\tau\in
  K_{\graph'}(a)$,
  and thus if $\tau\in K_{\graph'}(a)$, then $\iso^{-1}(\tau)\in K_{\graph}(a)$
  and $\iso_a(\iso^{-1}(\tau)) = \tau$.
  Having prove that $\phi_a$ is an isomorphism, let
  $C_\bullet(K_{\graph}(a))$ and $C_\bullet(K_{\graph'}(a))$ be the
  chain complexes induced by the simplicial complexes $K_{\graph}(a)$
  and $K_{\graph'}(a)$, respectively.
  The previous function $\iso_a$ induces a chain map $C_\bullet(\iso_a)$ between 
  $C_\bullet(K_{\graph}(a))$ and $C_\bullet(K_{\graph'}(a))$ defined as the linear map
  $C_k(\iso_a)$ sending an element $\sigma\in C_k(K_{\graph}(a))$ to
  $\iso_a(\sigma)$ for $k\geq 0$. 
  The proof that $C_\bullet(\iso_a)$ is a chain map can be found 
  in~\textcite[Proposition~4.5]{Nanda21}.
  Concretely, $C_k(\iso_a)$ is an isomorphism because $\iso_a$ is
  a simplicial complex isomorphism that generates a one-to-one
  correspondence between the $k$-simplices of $K_{\graph}(a)$ and
  $K_{\graph'}(a)$. As $C_\bullet(\iso_a)$ is a chain isomorphism,
  $C_\bullet(\iso_a)$ induces isomorphisms between the homology groups
  $\homologygroup{k}(C_\bullet(K_{\graph}(a)))$ and $\homologygroup{k}(C_\bullet(K_{\graph'}(a)))$ for all $k\geq 0$.
  Moreover, these isomorphisms $C_\bullet(\iso_a)$ constitute
  isomorphisms between the persistence modules given by
  $(\homologygroup{k}(C_\bullet(K_{\graph}(a))))_{a\in\reals}$ and
  $(\homologygroup{k}(C_\bullet(K_{\graph'}(a))))_{a\in\reals}$ for all
  $k\geq 0$. To prove it, we only need to show that the diagram
  \begin{equation}
    \begin{tikzcd}[ampersand replacement=\&]
      \homologygroup{k}(C_\bullet(K_{\graph}(a_i))) \arrow[rr, "i", hook] \arrow[d, "C_k(\iso_{a_i})"] \&  \& \homologygroup{k}(C_\bullet(K_{\graph}(a_j))) \arrow[d, "C_k(\iso_{a_j})"] \\
      \homologygroup{k}(C_\bullet(K_{\graph'}(a_i))) \arrow[rr, "i", hook]                             \&  \& \homologygroup{k}(C_\bullet(K_{\graph'}(a_j)))                            
    \end{tikzcd}
  \end{equation}
  commutes for $a_i\leq a_j$.
  This is a consequence of the definition of the isomorphisms $C_\bullet(\iso_a)$.
  Note that $C_k(\iso_{a_i})$ and $C_k(\iso_{a_j})$ are the same map for
  the elements of $C_k(K_{\graph}(a_i))$ for all $k\geq 0$ and $a_i\leq
  a_j$ due to the fact that $\iso_{a_i}=\iso_{a_j}$ for
  $K_{\graph}(a_i)$ for $a_i\leq a_j$ by definition. Thus, given $[c]\in
  \homologygroup{k}(C_\bullet(K_{\graph}(a_i)))$, we have
  $(C_k(\iso_{a_j})\circ i)([c]) = [\iso_{a_i}(c)] = (i\circ
  C_k(\iso_{a_i}))([c])$ for all $a_i\leq a_j$ and $k\geq 0$, as we
  wanted to prove.
\end{proof}

\kwlexpressivity*

\begin{proof}

  The main idea involves harnessing the colours of $k$-tuples.
  Denote by $\mathcal{F}(\graph) = (\simplicialcomplex_\graph, f_{\simplicialcomplex_\graph})$ 
  the output of our filtration generation. For a given graph $\graph$, we set 
  $\simplicialcomplex_\graph$ to be the simplicial complex of dimension zero with $0$-simplices 
  given by the vertices of $\graph$. Let $C_\graph$ be the multiset of colours assigned to 
  the graph $\graph$ by the \kWL algorithm. Without loss of generality, we assume that the \kWL 
  algorithm assigns colours that are always greater or equal than one.

  Now, given a finite multiset $S$ of colours, this is, a finite multiset of elements in 
  $\naturals \cap [2, +\infty)$, we define the representation of $S$ as the concatenation
  \begin{equation*}
  \text{R}(S) = c_1 \concat  \bigparallel_{i=2}^{|S|} 0 \concat r_9(c_i),
  \end{equation*}
  where $S=\multiset{c_1,\hdots, c_{|S|}}$ is the multiset of colours ordered in non-decreasing order, 
  and $r_9(c_i)$ is the representation of the number $c_i$ in bijective base-$9$ numeration. 
  The map from the set of multisets of colours to the set of natural numbers given by $R$ is injective 
  because the representation of individual natural numbers in bijective base-$9$ is injective 
  and because the digit zero does not appear in the base-$9$ representation of the numbers, allowing for 
  the separation of the different colours, and for the recovery of the original multiset.

  Set now $f_{\simplicialcomplex_\graph}$ to be the constant filtration function that assigns to each 
  simplex the value $\text{R}(C_\graph)$. Thus, the filtration generator $\mathcal{F}$ is equivariant 
  because the \kWL algorithm outputs the same multiset of colours for isomorphic graphs $G \simeq G'$, 
  and then the representation of the multiset of colours is the same for both graphs and yield the same 
  constant filtration function. 

  Finally, take two non-isomorphic graphs $\graph$ and $\graph'$. By definition of $\mathcal{F}$, the 
  simplicial complexes $\simplicialcomplex_\graph$ and $\simplicialcomplex_{\graph'}$ contain only vertices, 
  yielding non-empty persistence diagrams only in dimension zero. Particularly, the zero-dimensional persistence 
  diagrams of $\graph$ and $\graph'$ are multisets containing as many non-diagonal points as the number of
  vertices in the graphs of the form $(\text{R}(C_\graph), +\infty)$ and $(\text{R}(C_{\graph'}), +\infty)$, 
  respectively. Since the graphs are non-isomorphic, the multisets and their representations are different 
  and thus the persistence diagrams are different.
\end{proof}

\begin{restatable}{theorem}{onewlexpressivity}
  Given \WL colourings of two graphs $\graph$ and $\graph'$ that are
  different, there exists a filtration of $\graph$ and $\graph'$ such
  that their persistence diagrams in dimension~$0$ are also different.
  \label{thm:1-WL expressivity}
\end{restatable}
\begin{proof}
  Since the colourings are different, there is an iteration~$h$ of \WL
  such that the label sequences of $\graph$ and $\graph'$ are
  different. We thus have at least one colour---equivalently, one
  label---whose count is different.
  Let $\mathcal{L}^{(h)} := \{l_1, l_2, \dots\}$ be an enumeration of
  the finitely many hashed labels at iteration $h$. We can build
  a  filtration function~$f$ by assigning a vertex~$v$ with label
  $l_i$ to its index, i.e.\ $f(v) := i$, and setting $f(v, w) :=
    \max\left\{f(v), f(w)\right\}$ for an edge~$(v, w)$. The resulting
  \mbox{$0$-dimensional} persistence diagrams for $\graph$ and
  $\graph'$, denoted by $\diagram_0$ and $\diagram_0'$, respectively,
  now contain tuples of the form~$(i, j)$. Moreover, each vertex is
  guaranteed to give rise to \emph{exactly} one such pair since each
  vertex creates a connected component in \mbox{$0$-dimensional}
  persistent homology.
  Letting $\multiplicity{i,j}\mleft(\diagram_0\mright)$ refer to the
  multiplicity of a tuple in~$\diagram_0$, we know that, since the label
  count is different, there is \emph{at least} one tuple $(k, l)$ with
  $\multiplicity{k,l}\mleft(\diagram_0\mright) \neq
    \multiplicity{k,l}\mleft(\diagram_0'\mright)$.
  Hence, $\diagram_0 \neq \diagram_0'$.
\end{proof}

Prior to stating \cref{lem:CFITriangles} concerning CFI graphs, we must
first provide a precise definition of these structures. These
definitions can also be found in~\textcite[Section~6]{Cai92a}.
CFI graphs are constructed from a family of graphs $\{X_k\}_{k\geq 2}$. 
For each $k \geq 2$, we construct a pair of non-isomorphic CFI 
graphs $(\graph_k, H_k)$ such that $\graph_k$ and $H_k$ cannot be distinguished 
by the \kWL[(k-1)] algorithm but can be distinguished by the \kWL[k] algorithm.
Concretely, $X_k$ is defined as the graph $(V_k, E_k)$ where
\begin{inparaenum}[(i)]
  \item $V_k=A_k\cup B_k\cup M_k$ such that $A_k = \{a_i\}_{i=1}^k$,
  $B_k = \{b_i\}_{i=1}^k$, and $M_k = \{m_S:S\subseteq\{1,\hdots,k\}, 
  |S|\text{ is even}\}$, and
  \item $E_k = \{(m_S, a_i):i\in S\} \cup \{(m_S, b_i): i\not\in S\}$.
\end{inparaenum} 
Given any finite and connected graph $\graph$ with at least degree two for
all vertices, we can use the graphs $X_k$ to build new graphs 
$X(\graph)$ and $\tilde{X}(\graph)$ that allows us to construct the CFI 
graphs. $X(\graph)$ is built as follows. For each vertex $v$ of $\graph$,
we replace $v$ by a copy of $X_{\degree{v}}$, called $X(v)$. Then, for each
edge $\{u, v\}$ of $\graph$, we associate to each of its endpoints $u$ and $v$
one of the pairs $\{a_i, b_i\}$ from $X(u)$ and $X(v)$, denoting the pairs as
$\{a(\{u, v\}, u), b(\{u, v\}, u)\}$ and $\{a(\{u, v\}, v), b(\{u, v\}, v)\}$, 
respectively. Finally, for each of edges $\{u,v\}$ of $\graph$, we add the edges
$\{a(\{u, v\}, u), a(\{u, v\}, v)\}$ and $\{b(\{u, v\}, u), b(\{u, v\}, v)\}$ 
to $\graph$,
discarding the original edge. The graph $\tilde X(\graph)$ is constructed 
in the same way as $X(\graph)$, but we arbitrarily choose one edge 
$\{u,v\}$ of $\graph$ and, instead of adding the previous two edges, we add
the edges $\{a(\{u, v\}, u), b(\{u, v\}, v)\}$ and $\{a(\{u, v\}, v), 
b(\{u, v\}, u)\}$.
Finally, $G_k$ and $H_k$ are given by $G_k = X(T_k)$ and $H_k = \tilde
X(T_k)$ where $T_k$ is a degree-three graph with separator size $k$. In
the original construction, these graphs are also coloured, but we will
not need this information for the proof of \cref{lem:CFITriangles}.

\begin{restatable}{lemma}{CFItrianglesGeneral}
  For any finite and connected graph $\graph$, $X(\graph)$ and
  $\tilde{X}(\graph)$ have no cycles of length $3$, and thus, no cliques
  of size $3$ or more.
  \label{lem:CFITrianglesGeneral}
\end{restatable}
\begin{proof}
  Let $\cycle$ be a cycle of $X(\graph)$ of length three. 
  First, note that there are no cycles of length $3$ in the 
  graphs $X_k$ for any $k\geq 1$. This
  is because all edges go from $M_k$ to a subset $A_k$ or $B_k$. 
  This implies that, in order to have a cycle, one needs at least 
  $4$ edges, passing twice by the set $M_k$.
  Therefore, the cycle $\cycle$ must contain at least one edge $e$
  connecting two subgraphs $X(u)$ and $X(v)$ for $u$ and $v$ original 
  vertices of $\graph$. The endpoints of this edge must be connected
  by design to two disjoint sets of type $M_k$ of the subgraphs 
  $X(u)$ and $X(v)$, respectively. However, if the cycle $\cycle$ is
  of length three, this means that the two endpoints of $e$ are connected
  to the same element, implying that both sets $M_k$ are not disjoint and
  arriving at a contradiction. The proof for $\tilde X(\graph)$ is analogous.
\end{proof}

\begin{restatable}{corollary}{CFItriangles}
  For any $k\geq 2$, the CFI graphs
  $G_k$ and $H_k$ have no cycles of length $3$, and 
  thus, no cliques of size $3$ or more.
  \label{lem:CFITriangles}
\end{restatable}
\begin{proof}
  This is a direct consequence of \cref{lem:CFITrianglesGeneral}. 
\end{proof}

\begin{figure}
  [tbp]
  \centering
  \includegraphics[width=\textwidth]{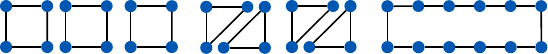}
  \caption{The image shows three graphs composed of three disjoint 
  cycles of length four, four disjoint cycles of length three, and a
  single cycle of length twelve. For \WL, all vertices are initially
  assigned the same colour, in this case blue. During the first iteration, 
  the \texttt{RELABEL} function for each vertex receives the tuple 
  $(\textcolor{blue}{\bullet}, \multiset{\textcolor{blue}{\bullet}, \textcolor{blue}{\bullet}})$. 
  Consequently, all vertices in the three graphs receive the same new colour, 
  keeping intact the initial partitions, finishing the algorithm. 
  Thus, the final multiset of colours is the same for all three cases
  (twelve blue colours), making \WL unable to distinguish between these three 
  graphs.}
  \label{fig:connected_components_wl}
\end{figure}

\begin{restatable}{theorem}{connected_components_wl}
  Let $n=ab$ with integers $a,b\geq 3$. Then, the \WL test cannot
  distinguish between $a$ cycles of length $b$, $b$ cycles of length
  $a$, and a single cycle of length $n$.
  \label{lem:Connected components WL}
\end{restatable}
\begin{proof}
  In all three cases, each graph consists of $n$ vertices, with each vertex 
  having a degree of two. Initially, all vertices across the three graphs 
  are assigned the same colour $c$. During the first iteration, each node 
  receives as input the tuple $\left(c, \{\{c, c\}\}\right)$. 
  Consequently, after this iteration, all vertices in the three graphs 
  obtain the same colour assignment.
  Since all vertices have the same colour at the end of the first iteration, 
  the algorithm terminates. This occurs because the partition of nodes 
  generated by the colours after the first iteration is identical to the partition 
  from the initial assignment. As a result, the \WL test fails to distinguish 
  between the three cases.
\end{proof}

An example for $a=3$, $b=4$ is provided in~\cref{fig:connected_components_wl}. 
We end this section with two results concerning \emph{other} graph
properties that are being captured by persistent homology.

\diameterbound*

\begin{proof}
  We can provide a procedure to obtain an upper bound~$d$ of
  $\diam(\graph)$ alongside the calculation of $\diagram_0$. To this
  end, we set $d = 0$. While calculating~$\diagram_0$ with
  \cref{alg:Connected components}, we check for each edge whether it
  is a destroyer, or a regular edge. If the edge is a destroyer,
  we increase $d$ by one; otherwise, we do nothing.
  This procedure works because $\diam(\graph)$
  is upper bounded by the diameter of the minimum spanning tree of~$\graph$. Our
  estimate~$d$ counts the number of edges in such a tree. We thus have
  $\diam(\graph) \leq d$. The bound is tight for some graphs, such as
  line graphs.
\end{proof}

\girthbound*

\begin{proof}
  Similar to \cref{thm:Diameter}, we can use the calculation of
  $\diagram_0$, the persistence diagram in dimension~$0$, of the
  filtration function to obtain an upper bound of the girth of the
  graph. We calculate~$\diagram_0$ with
  \cref{alg:Connected components}, checking once again for each edge
  whether it is a creator or destroyer. If the edge is a
  creator, we
  stop and set our upper bound of the girth~$g$ to be the number of
  vertices in the connected component
  that contains the endpoints of the edge.
  Since the addition of the respective edge is
  guaranteed to create a cycle---the edge is a creator---the cycle
  has to make use of at most as many vertices as the number of vertices
  in the connected component to which it belongs.
  Our estimate~$g$ thus constitutes an upper bound of the
  girth of the graph.
\end{proof}

\Cref{thm:Diameter,thm:Girth} indicate that persistent
homology captures more than `just' topological information about
a graph. We substantiate these theorems with empirical results
concerning different filtrations and other graph properties in
\cref{sec:Predicting Graph Properties}, thus showing how a topological perspective
complements and enriches graph-learning tasks.

\section{Topology and the Weisfeiler--Leman Hierarchy}\label{sec:Topology and WL}

In the following, we want to briefly extend the argumentation of the
expressivity of persistent homology with respect to the
Weisfeiler--Leman hierarchy.
Since \WL is oblivious to certain topological structures such as
cycles~\autocite{Arvind15a}, the existence of graphs with different
Betti number counts proves that persistent homology is \emph{strictly
  more expressive} than \WL. For example, consider a graph consisting of
the union of two triangles, i.e.\
\begin{tikzpicture}[baseline]
  \begin{scope}[scale=0.25]
    \coordinate (A) at (0.0, 0.5);
    \coordinate (B) at (0.5, 0.0);
    \coordinate (C) at (0.5, 1.0);

    \coordinate (D) at (1.0, 1.0);
    \coordinate (E) at (1.5, 0.5);
    \coordinate (F) at (1.0, 0.0);
  \end{scope}

  \foreach \c in {A,B,...,F}
    {
      \filldraw[black] (\c) circle (1pt);
    }

  \draw (A) -- (B) -- (C) -- cycle;
  \draw (D) -- (E) -- (F) -- cycle;
\end{tikzpicture}. This graph has $\betti{0} = \betti{1} = 2$ since it
consists of two connected components and two cycles. If we change the
connectivity slightly to obtain a hexagon, i.e.\
\begin{tikzpicture}[baseline]
  \begin{scope}[scale=0.25]
    \coordinate (A) at (0.0, 0.5);
    \coordinate (B) at (0.5, 0.0);
    \coordinate (C) at (0.5, 1.0);

    \coordinate (D) at (1.0, 1.0);
    \coordinate (E) at (1.5, 0.5);
    \coordinate (F) at (1.0, 0.0);
  \end{scope}

  \foreach \c in {A,B,...,F}
    {
      \filldraw[black] (\c) circle (1pt);
    }
  \draw (A) -- (C) -- (D) -- (E) -- (F) -- (B) -- cycle;
\end{tikzpicture}, we obtain a graph with $\betti{0} = \betti{1} = 1$.
\WL is not able to distinguish between these graphs, but persistent
homology can, since the Betti numbers of the graph are still encoded
in the persistence diagram as essential features.
Note that \cref{thm:1-WL expressivity} does \emph{not} apply to
arbitrary filtrations since the theorem requires knowing the correct labels
assigned by $\WL$. Finding filtration functions that are able to split
graphs in a manner that is provably equivalent to $\WL$ remains an open
research question.

\section{Additional Expressivity Experiments}\label{app:Expressivity Experiments}

This section contains additional expressivity experiments that had to be
excluded from the main paper for reasons of space.

\subsection{Additional Results for Connected Cubic Graphs}

We start our supplementary experiments by distinguishing connected cubic graphs,
i.e.\ \mbox{$3$-regular} graphs. These graphs cannot be distinguished
by \WL, but they can be distinguished by \kWL[2]~\autocite{Bollobas82a, Morris21a}.
As such, they provide a good example of how different filtrations
harness different types of graph information. \cref{tab:Connected
  cubic graphs} shows the results. We first observe that the degree
filtration is incapable of distinguishing graphs for $k=1$. This is
a direct consequence of the regularity---since the function is
constant on the graph, persistent homology cannot capture any
variability. This changes, however, when higher-order
structures---triangles---are included for $k=2$. We also observe that
the Laplacian-based filtration for $k=2$ exhibits strong empirical
performance; in the absence of additional information, the spectral
properties captured by the Laplacian help in distinguishing graphs.
The Ollivier--Ricci curvature filtration is performing similarly well also for
the same dimension, while it outperforms the Laplacian filtration for $k=1$.
In contrast to the Laplacian-based filtration, it does not require the
calculation of eigenvalues, which may be prohibitive for larger
graphs.\footnote{%
  Ollivier--Ricci curvature requires solving optimal transport
  problems, for which highly efficient approximative solvers are
  available~\autocite{Cuturi13a}.
}
Finally, we see that the Forman--Ricci curvature filtration, a purely
combinatorial quantity depending only on local counts, is consistently
outperformed by the Ollivier--Ricci curvature for $k=1$. However, for $k=1$,
the Forman--Ricci filtration outperforms the Laplacian filtration.
Finally, the Vietoris--Rips filtration is incapable of distinguishing the
graphs for $k=1$, and performs similarly to the Forman--Ricci curvature for
$k=2$, suggesting that filtrations based on distances are not capable of
capturing the graph structure in an optimal way.

\begin{table}[btp]
  \sisetup{
    table-format    = 1.2,
    round-mode      = places,
    round-precision = 2,
    detect-all      = true,
    detect-weight   = true,
  }
  \centering
  \caption{%
    Success rate~($\uparrow$) of distinguishing pairs of
    \emph{connected cubic graphs} when using five different
    filtrations at varying expansion levels
    of the graph~(denoted by~$k$). Due to the regularity of each
    graph, $k = 3$ is omitted since the clique complex of the graph is
    exactly the same as for $k = 2$. These graphs can be distinguished
    by \kWL[2] but not by \WL.
  }%
  \label{tab:Connected cubic graphs}
  \smallskip
  \begingroup
  \let\b=\bfseries
  \let\e=\itshape
  \setlength{\tabcolsep}{4pt}
  \begin{tabular}{lSSSSSSSSSS}
    \toprule
    \multirow{5}{*}{\e Data} & \multicolumn{5}{c}{$k = 1$}     & \multicolumn{5}{c}{$k = 2$}                       \\
    \cmidrule(lr){2-11}
                             & \multicolumn{10}{c}{\e Filtration}                                                  \\
    \cmidrule(lr){2-11}
                             & {D}   & {O}     & {F}    & {L}  & {V}  & {D}    & {O}    & {F}    & {L} & {V}       \\
    \midrule
    \texttt{cub06}           & 0.00  & \b1.00  & \b1.00 & 0.00 & 0.00 & \b1.00 & \b1.00 & \b1.00 & \b1.00 & \b1.00 \\
    \texttt{cub08}           & 0.00  &   0.90  &   0.70 & 0.00 & 0.00 &   0.90 &   0.90 &   0.90 & \b1.00 &   0.90 \\
    \texttt{cub10}           & 0.00  &   0.98  &   0.66 & 0.00 & 0.00 &   0.81 &   0.99 &   0.88 & \b1.00 &   0.85 \\
    \texttt{cub12}           & 0.00  &   0.99  &   0.64 & 0.00 & 0.00 &   0.80 &   1.00 &   0.87 & \b1.00 &   0.87 \\
    \texttt{cub14}           & 0.00  &   0.99  &   0.62 & 0.00 & 0.00 &   0.79 &   1.00 &   0.86 & \b1.00 &   0.89 \\
    \bottomrule
  \end{tabular}
  \endgroup
\end{table}

\subsection{Additional Results for the BREC Data Set}\label{sec:Additional Results for the BREC Data Set}

\begin{table}[tbp]
  \centering
  \caption{
    Success rate~($\uparrow$) for distinguishing pairs of instances of the
    \emph{BREC data set} when using five different filtrations at varying
    expansion levels of the graph~(denoted by~$k$). Due to combinatorial
    constraints, we did not calculate the Vietoris--Rips filtration for $k=4$.
    Legend and number of graphs per category: \texttt{B}~(Basic, 60), \texttt{R}~(Regular, 100),
    \texttt{E}~(Extension, 100), \texttt{C}~(CFI, 100), \texttt{4, 20}~($4$-Vertex
    Condition), \texttt{D}~(Distance-Regular, 20) graphs, respectively
    and \texttt{A}~(average over full data set, 400).
  }
  \label{tab:BREC itemised}
  \smallskip
  \sisetup{
    table-format    = 1.2,
    round-mode      = places,
    round-precision = 2,
    detect-all      = true,
    detect-weight   = true,
  }
  \begingroup
  \let\b=\bfseries
  \let\e=\itshape
  \setlength{\tabcolsep}{2pt}
  \begin{tabular}{lSSSSS}
    \toprule
    & \multicolumn{5}{c}{$k = 1$}                                         \\
    \midrule\multirow{2.5}{*}{\e Data}
                             & \multicolumn{5}{c}{\e Filtration}        \\
    \cmidrule(lr){2-6}
                             & {D}   & {O}    & {F}   & {L}     & {V}   \\
    \midrule                                                                
    \texttt{Basic} (60)      & 0.033 & 0.933  & 0.867 & \b1.000 & 0.000 \\
    \texttt{Regular} (100)   & 0.000 & 0.420  & 0.320 &   0.000 & 0.000 \\
    \texttt{Extension} (100) & 0.070 & 0.760  & 0.440 &   0.940 & 0.000 \\
    \texttt{CFI} (100)       & 0.030 & 0.030  & 0.030 & \b0.060 & 0.030 \\
    \texttt{4-VC} (20)       & 0.000 & 0.000  & 0.000 &   0.000 & 0.000 \\
    \texttt{DR} (20)         & 0.000 & 0.000  & 0.000 & \b0.050 & 0.000 \\
    Average (400)            & 0.030 & 0.443  & 0.328 &   0.403 & 0.008 \\
    \midrule
                             & \multicolumn{5}{c}{$k = 2$}                                 \\
    \midrule\multirow{2.5}{*}{\e Data}
                             & \multicolumn{5}{c}{\e Filtration}        \\
    \cmidrule(lr){2-6}
                             & {D}   & {O}     & {F}   & {L}     & {V}  \\
    \midrule                                                            
    \texttt{Basic} (60)      & 0.783 & \b1.000 & 0.983 & \b1.000 & 0.517\\
    \texttt{Regular} (100)   & 0.390 &   0.540 & 0.500 &   0.480 & 0.390\\
    \texttt{Extension} (100) & 0.260 &   0.920 & 0.590 & \b1.000 & 0.110\\
    \texttt{CFI} (100)       & 0.030 &   0.030 & 0.030 & \b0.060 & 0.030\\
    \texttt{4-VC} (20)       & 0.000 &   0.000 & 0.000 &   0.000 & 0.000\\
    \texttt{DR} (20)         & 0.000 &   0.000 & 0.000 & \b0.050 & 0.000\\
    Average (400)            & 0.287 &   0.522 & 0.427 &   0.537 & 0.210\\
    \midrule 
                             & \multicolumn{5}{c}{$k = 3$}                                                                                           \\
    \midrule\multirow{2.5}{*}{\e Data}
                             & \multicolumn{5}{c}{\e Filtration}                                                                     \\
    \cmidrule(lr){2-6}
                             & {D}     & {O}     & {F}     & {L}     & {V}     \\
    \midrule
    \texttt{Basic} (60)      &   0.833 & \b1.000 & 0.983   & \b1.000 &   0.583 \\
    \texttt{Regular} (100)   &   0.850 &   0.930 & 0.910   &   0.930 &   0.850 \\
    \texttt{Extension} (100) &   0.290 &   0.920 & 0.590   & \b1.000 &   0.160 \\
    \texttt{CFI} (100)       &   0.030 &   0.030 & 0.030   & \b0.060 &   0.030 \\
    \texttt{4-VC} (20)       & \b1.000 & \b1.000 & \b1.000 & \b1.000 & \b1.000 \\
    \texttt{DR} (20)         &   0.000 &   0.000 &   0.000 & \b0.050 & \b0.050 \\
    Average (400)            &   0.468 &   0.670 &   0.580 &   0.700 &   0.400 \\
    \midrule 
                             & \multicolumn{5}{c}{$k = 4$}                                                                                           \\
    \midrule\multirow{2.5}{*}{\e Data}
                             & \multicolumn{5}{c}{\e Filtration}                                                                     \\
    \cmidrule(lr){2-6}
                             & {D}      & {O}     & {F}     & {L}     & {V}  \\
    \midrule
    \texttt{Basic} (60)      &   0.833  & \b1.000 &   0.983 & \b1.000 & {---}\\
    \texttt{Regular} (100)   &   0.890  &   0.970 &   0.950 &   0.970 & {---}\\
    \texttt{Extension} (100) &   0.290  &   0.920 &   0.590 & \b1.000 & {---}\\
    \texttt{CFI} (100)       &   0.030  &   0.030 &   0.030 &   0.060 & {---}\\
    \texttt{4-VC} (20)       & \b1.000  & \b1.000 & \b1.000 & \b1.000 & {---}\\
    \texttt{DR} (20)         &   0.000  &   0.000 &   0.000 &   0.050 & {---}\\
    Average (400)            &   0.477  &   0.680 &   0.590 &   0.710 & {---}\\
    \bottomrule
  \end{tabular}
  \endgroup
\end{table}

\cref{tab:BREC itemised} shows BREC data set results, itemised by the
value of~$k$, while \cref{tab:BREC summary} provides a summary and
performance comparison of the persistent homology results with other
baselines. 

\begin{table}[tbp]
  \sisetup{
      table-format    = 1.2,
      round-mode      = places,
      round-precision = 2,
      detect-all      = true,
      detect-weight   = true,
  }
  \centering
  \caption{
      Success rate~($\uparrow$) for distinguishing pairs of instances of the
      \emph{BREC data set} when using four different
      filtrations for $k=4$. Vietoris--Rips filtration not computed for
      $k=4$ due to computational constraints. The data sets \texttt{Regular},
      \texttt{4-Vertex Condition}, and \texttt{Distance Regular}, from~\cref{tab:BREC}
      are merged into the \texttt{All regular} data set. \textcolor{ForestGreen}{Green} indicates 
      the best performing algorithm, while \textcolor{orange}{orange} indicates the second best.
  }
  \label{tab:BREC summary}
  \vspace{4pt}
  \begingroup
  \let\b=\bfseries
  \setlength{\tabcolsep}{4pt}
  \begin{tabular}{l S S S S S S }
      \toprule
      \multirow{2}{*}{Data}      & \multicolumn{2}{c}{SOTA} & \multicolumn{4}{c}{Filtration ($k=4$)}                                                                                     \\
      \cmidrule(lr){2-3} \cmidrule(lr){4-7}
                                 & {$3$-WL}                 & {$N_2$}                                & {D}   & {O}                   & {F}                   & {L}                   \\
      \midrule
      \texttt{Basic} (60)        & {\color{ForestGreen}1.000}       & {\color{ForestGreen}1.000}                     & 0.833 & {\color{ForestGreen}1.000}    & {\color{orange}0.983} & {\color{ForestGreen}1.000}    \\
      \texttt{All regular} (140) & 0.357                    & {\color{ForestGreen}0.986}                     & 0.779 & 0.836                 & 0.821                 & {\color{orange}0.843} \\
      \texttt{Extension} (100)   & {\color{ForestGreen}1.000}       & \color{ForestGreen}{1.000}                     & 0.290 & {\color{orange}0.920} & 0.590                 & {\color{ForestGreen}1.000}    \\
      \texttt{CFI} (100)         & {\color{ForestGreen}0.600}       & 0.000                                  & 0.030 & 0.030                 & 0.030                 & {\color{orange}0.060} \\
      \midrule
      Average (400)              & 0.675                    & {\color{ForestGreen}0.745}                     & 0.477 & 0.680                 & 0.590                 & {\color{orange}0.710}\\
      \bottomrule
  \end{tabular}
  \endgroup
\end{table}

\subsection{Additional Figures for Graph-Property Prediction Tasks}

\cref{fig:ogbg-molhiv_radii} shows the distributions of graph properties
that we predict in the main paper in \cref{sec:Predicting Graph
Properties}.

\begin{figure}[tbp]
  \centering
  \includegraphics[width=\textwidth]{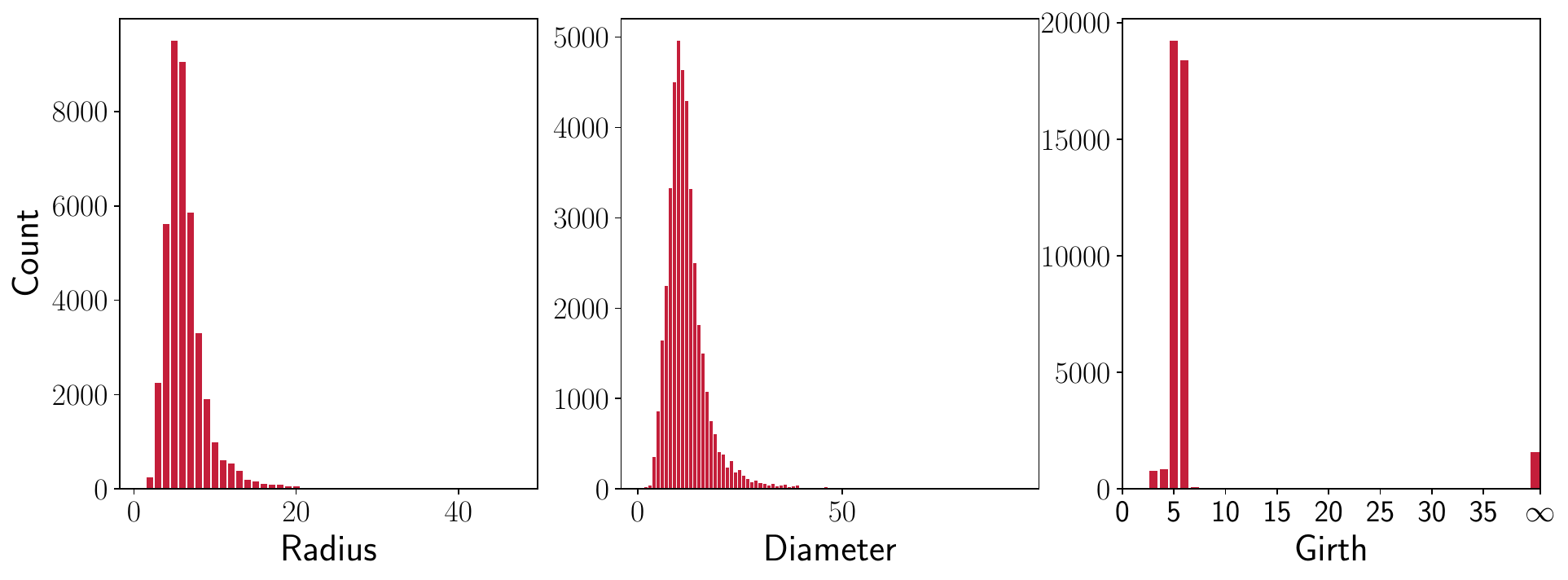}
  \caption{
    Distribution of maximum radii, maximum diameters, and girths in the 
    \texttt{ogbg-molhiv} molecular graph
    data set~\autocite{hu20a}. The values are concentrated on the lower
    end of the spectrum, with medians of~$6$, $11$, and $5$ and maximums of~$47$,
     $93$, and $36$, respectively. The maximum value for the girth is computed
      without taking into account infinite values.}
  \label{fig:ogbg-molhiv_radii}
\end{figure}

\section{Additional Graph Classification Experiments}\label{app:Graph Classification Experiments}

\begin{table*}[tbp]
  \sisetup{
    table-format    = 2.2(2),
    round-mode      = places,
    round-precision = 2,
    detect-all      = true,
    detect-weight   = true,
    separate-uncertainty = true,
    retain-zero-uncertainty=true
  }
  \centering
  \caption{
    Graph classification average accuracy~($\uparrow$) and standard deviation
    in test for the experiments with data set \texttt{IMDB-Binary} (\texttt{I}),
    \texttt{PROTEINS} (\texttt{P}), \texttt{MUTAG} (\texttt{M}), and \texttt{NCI1} 
    (\texttt{N}), and  ROC-AUC ($\uparrow$)
    for the \texttt{ogbg-molhiv} (\texttt{O}) test data set. The results are averaged over
    three runs using a stratified $3$-means for all data set except for
    \texttt{ogbg-molhiv}, where only one run is performed. The best
    results are highlighted in bold. Experiments with the Vietoris--Rips
    filtration for the \texttt{PROTEINS} and \texttt{molhiv} are not reported
    for $k=2$ due to out-of-resources errors during execution. $S_1$ and $S_2$
    refer to accuracy of the most effective methods reported in~\autocite{OBray21a} 
    and~\autocite{Hofer20}, respectively. The abbreviation DS stands for data set.
    }
  \label{tab:Classification results}
  \begingroup
  \let\b=\bfseries
  \setlength{\tabcolsep}{2pt}
  \begin{tabular}{l S S S S S S S}
    \toprule
    \multicolumn{8}{c}{$k = 1$}                                                                                                                  \\
    \midrule\multirow{2}{*}{DS}
    &  \multicolumn{2}{c}{SOTA}    &\multicolumn{5}{c}{Filtration}\\
    \cmidrule(lr){2-3}\cmidrule(lr){4-8}
    & {$S_1$} & {$S_2$} & {D}                              & {O}                        & {F}                        & {L}               & {V}                 \\
    \midrule
    \texttt{P} & \b0.75 \pm 0.05 & 0.74 \pm 0.03 & 0.70 \pm 0.01                & 0.74 \pm 0.01 & 0.72 \pm 0.02          & 0.72 \pm 0.02 & 0.70 \pm 0.00 \\
    \texttt{I} & 0.74 \pm 0.04 & \b0.75 \pm 0.05 & \b0.75 \pm 0.02       & 0.71 \pm 0.01          & 0.69 \pm 0.03          & 0.65 \pm 0.01 & 0.68 \pm 0.03 \\
    \texttt{M}    & 0.87 \pm 0.08 & {-} & 0.86 \pm 0.01                & 0.87 \pm 0.04          & \b0.90 \pm 0.04 & 0.79 \pm 0.06 & 0.87 \pm 0.01 \\
    \texttt{N}     & {-} & 0.71 \pm 0.02 & 0.68 \pm 0.00                & 0.73 \pm 0.00          & 0.69 \pm 0.01          & 0.67 \pm 0.00 & 0.66 \pm 0.01 \\
    \texttt{O}   & {-} & {-} & \b0.50      & 0.50          & \b0.50 & 0.50 & 0.50 \\
    \midrule\multicolumn{8}{c}{$k = 2$}                                                                                                          \\
    \midrule\multirow{2}{*}{DS}
    &  \multicolumn{2}{c}{SOTA}    &\multicolumn{5}{c}{Filtration}\\
    \cmidrule(lr){2-3}\cmidrule(lr){4-8}
    & {$S_1$} & {$S_2$} & {D}                              & {O}                        & {F}                        & {L}               & {V}                 \\
    \midrule
    \texttt{P} & 0.75 \pm 0.05 & 0.74 \pm 0.03 & 0.70 \pm 0.02                & 0.73 \pm 0.02          & 0.70 \pm 0.01          & 0.68 \pm 0.01 & {-}               \\
    \texttt{I} & 0.74 \pm 0.04 & 0.75 \pm 0.05 & 0.73 \pm 0.03                & 0.71 \pm 0.02          & 0.72 \pm 0.03          & 0.68 \pm 0.03 & 0.66 \pm 0.02 \\
    \texttt{M}    & 0.87 \pm 0.08 & {-} & 0.86 \pm 0.01                & 0.89 \pm 0.04          & 0.89 \pm 0.03          & 0.81 \pm 0.04 & 0.87 \pm 0.03 \\
    \texttt{N}     & {-} & 0.71 \pm 0.02  & 0.68 \pm 0.00                & \b0.74 \pm 0.01 & 0.70 \pm 0.00          & 0.68 \pm 0.00 & 0.69 \pm 0.01 \\
    \texttt{O}   & {-} & {-} & \b0.50       & \b0.50 & 0.50          & 0.50 & {-}               \\
    \bottomrule
  \end{tabular}
  \endgroup
\end{table*}
To assess the capacity of persistent homology in graph learning tasks,
we have designed a set of graph classification experiments that use the
filtrations introduced in~\cref{sec:Experiments} to extract persistence diagrams
from graph classification data set to perform inference on them.
Two fundamental differences between our approach and the one used in~\autocite{Hofer20} are that, 
we do not learn an optimal filtration and also we do not use deep learning models to 
perform the classification, as we are interested into the capacity of persistent homology 
to perform classification, and not in its combination with neural networks. For this 
endeavour, we train a Random Forest classifier on the persistent 
images~\autocite{Adams17a} computed from the persistence diagrams extracted from 
the input graphs using the previous filtrations up to dimension $k=2$. We test our 
approach in the \texttt{MUTAG}~\autocite{Debnath91a}, 
\texttt{IMDB-Binary}~\autocite{Yanardag15a}, \texttt{PROTEINS}~\autocite{Borgwardt05a}, 
\texttt{NC1}~\autocite{Wale06a}, and \texttt{ogbg-molhiv}~\autocite{hu20a} data set. 
Except for \texttt{ogb-molhiv}, we perform a Stratified $3$-Fold experiment and provide 
the average and standard deviation of the multiple runs. For \texttt{ogbg-molhiv}, we 
perform only one experiment instance with the official train and test splits.
The results are shown in~\cref{tab:Classification results}. 
We observe that, although we
discard the data set features when using persistent homology, we are able to achieve
suitably good results on some of the data sets, such as \texttt{MUTAG} and
\texttt{IMDB-Binary}~\autocite{errica20a}, even surpassing state-of-the-art results reported
in \textcite{Hofer20} and \textcite{OBray21a} for \texttt{IMDB-Binary}, \texttt{MUTAG}, 
and \texttt{NCI1}.
For the \texttt{ogbg-molhiv}, however, we obtain consistent ROC-AUC values
of $0.5$, which is the same as random guessing.
Our hypothesis is that molecular graphs rely on the annotated data set
features to perform well, and the isolated topological information is
not enough to perform the classification. The success of the experiments
suggests that persistent homology is capable of capturing essential
structural information about the graphs to classify and thus, suggests
that persistent homology can be expressive in practice. We leave a more
structured and comprehensive benchmark for persistent homology-like
methods for graph-learning tasks to future work.

\section{Additional results on alpha complexes}
\label{sec:Additional results on alpha complexes}
As the last set of experiments, we repeat the experiments from~\Cref{sec:Experiments}
except for the the property prediction tasks on random graphs, but using an alpha complex  
filtration. \emph{Alpha complex filtrations}~\autocite[Section~III.4]{edelsbrunner2010computational} are 
\emph{geometric} filtrations, meaning that they are computed from point clouds in the 
Euclidean space, in our case, coming from embeddings of graph vertices. Concretely, we embed 
the vertices of our graphs in $\reals^3$ using 
\emph{Laplacian eigenmaps}~\autocite{lap_eigenmaps}. Note that the embedding map is 
an arbitrary choice, and we leave the exploration of other embeddings for future work, making 
the use of alpha complexes in this context a proof of concept.
\Cref{tab:BREC alpha complexes,tab:Classification alpha complexes,tab:Connected cubic graphs alpha complexes,tab:Minimal Cayley graphs alpha complexes,tab:Strongly-regular graphs alpha complexes,tab:Radii prediction experiments alpha complexes} contain the BREC, classification, connected cubic, minimal Cayley, strongly-regular, and property prediction experiments, respectively, using the alpha complex filtration.

For BREC, we observe that the alpha complex filtration obtains almost perfect success rate 
in distinguishing the non-isomorphic graphs, failing only on the CFI category, where 
the alpha complex filtration still shows an strong success rate, outperforming the other 
filtrations and non-topological methods by a large margin. For the other isomorphism datasets (connected cubic, minimal Cayley, strongly-regular), alpha 
complexes obtain perfect accuracy in all cases.

For the classification tasks, though, we observe that the alpha complex filtration performs 
worse than the best filtrations at each task worse than the other state-of-the-art methods.
Due to the good results in the graph isomorphism tasks, we hypothesize that the alpha complex 
is highly expressive (and sensitive) and that the Random Forest with the persistence images 
overfits the data without further regularization. 

Finally, for the property prediction experiments, we observe that the alpha complex 
filtration, as in the classification tasks, performs worse than the other filtrations 
except for the girth prediction, where it obtains better results than curvature-based and 
degree filtrations.

Overall, our preliminary results suggest that the alpha complex filtration is highly expressive and capable of
capturing the graph structure, but it is also sensitive to noise and overfitting, at least 
for Laplacian eigenmaps embeddings.

\begin{table}[tbp]
  \centering
  \caption{
    Success rate~($\uparrow$) for distinguishing pairs of instances of the
    \emph{BREC data set} when using the alpha complex filtration described 
    at~\Cref{sec:Additional results on alpha complexes} and persistence diagrams up to 
    varying dimensions~(denoted by~$k$).
    Legend and number of graphs per category: \texttt{B}~(Basic, 60), \texttt{R}~(Regular, 100),
    \texttt{E}~(Extension, 100), \texttt{C}~(CFI, 100), \texttt{4, 20}~($4$-Vertex
    Condition), \texttt{D}~(Distance-Regular, 20) graphs, respectively
    and \texttt{A}~(average over full data set, 400). For the results using the other filtrations,
    see~\cref{tab:BREC itemised}.
  }
  \label{tab:BREC alpha complexes} 
  \smallskip
  \sisetup{
    table-format    = 1.2,
    round-mode      = places,
    round-precision = 2,
    detect-all      = true,
    detect-weight   = true,
  }
  \begingroup
  \let\b=\bfseries
  \let\e=\itshape
  \setlength{\tabcolsep}{2pt}
\begin{tabular}{l S S S}
  \toprule
  Data & \multicolumn{1}{c}{$k = 1$} & \multicolumn{1}{c}{$k = 2$} & \multicolumn{1}{c}{$k = 3$} \\
  \midrule
  Basic (60) & \b1.000 & \b1.000 & \b1.000 \\
  Regular (100) & \b1.000 & \b1.000 & \b1.000 \\
  Extension (100) & \b1.000 & \b1.000 & \b1.000 \\
  CFI (100) & 0.630 & 0.800 & \b0.870 \\
  4-Vertex\_Condition (20) & \b1.000 & \b1.000 & \b1.000 \\
  Distance\_Regular (20) & \b1.000 & \b1.000 & \b1.000 \\
  \cmidrule(lr){2-4}
  Average (400) & 0.907 & 0.950 & \b0.968 \\
  \bottomrule
  \end{tabular}
  \endgroup
\end{table}

\begin{table}[tbp]
  \centering
  \caption{Graph classification average accuracy~($\uparrow$) and standard deviation
  in test for the experiments with data set \texttt{IMDB-Binary} (\texttt{I}),
  \texttt{PROTEINS} (\texttt{P}), \texttt{MUTAG} (\texttt{M}), and \texttt{NCI1} 
  (\texttt{N}), and  ROC-AUC ($\uparrow$)
  for the \texttt{ogbg-molhiv} (\texttt{O}) test data set 
  when using the alpha complex filtration described 
    at~\Cref{sec:Additional results on alpha complexes} and persistence diagrams up to 
    varying dimensions~(denoted by~$k$). The results are averaged over
  three runs using a stratified $3$-means for all data set except for
  \texttt{ogbg-molhiv}, where only one run is performed. The best
  results are highlighted in bold. For the results using the other filtrations,
  see~\cref{tab:Classification results}.}
  \label{tab:Classification alpha complexes}
  \smallskip
  \sisetup{
    table-format    = 2.2(2),
    round-mode      = places,
    round-precision = 2,
    detect-all      = true,
    detect-weight   = true,
    separate-uncertainty = true,
    retain-zero-uncertainty=true
  }
  \begingroup
  \let\b=\bfseries
  \let\e=\itshape
  \setlength{\tabcolsep}{2pt}
\begin{tabular}{l S S}
  \toprule
  Data set & \multicolumn{1}{c}{$k = 1$} & \multicolumn{1}{c}{$k = 2$} \\
  \midrule
  \texttt{P} & \b0.66 \pm 0.02 & 0.66 \pm 0.01 \\
  \texttt{I} & \b0.59 \pm 0.04 & 0.57 \pm 0.04 \\
  \texttt{M} & 0.79 \pm 0.03 & \b0.79 \pm 0.03 \\
  \texttt{N} & \b0.62 \pm 0.01 & 0.60 \pm 0.02 \\
  \texttt{O} & 0.48 \pm 0.00 & \b0.50 \pm 0.00 \\
  \bottomrule
  \end{tabular}
  \endgroup
\end{table}

\begin{table}[tbp]
  \centering
  \caption{Success rate~($\uparrow$) of distinguishing pairs of
  \emph{connected cubic graphs} when using the alpha complex filtration described 
  at~\Cref{sec:Additional results on alpha complexes} and persistence diagrams up to 
  varying dimensions~(denoted by~$k$). These graphs can be distinguished
  by \kWL[2] but not by \WL. For the results using the other filtrations,
  see~\cref{tab:Connected cubic graphs}.}
  \label{tab:Connected cubic graphs alpha complexes}
  \smallskip
  \sisetup{
    table-format    = 1.2,
    round-mode      = places,
    round-precision = 2,
    detect-all      = true,
    detect-weight   = true,
  }
  \begingroup
  \let\b=\bfseries
  \let\e=\itshape
  \setlength{\tabcolsep}{2pt}
\begin{tabular}{l S S}
  \toprule
  Data & \multicolumn{1}{c}{$k = 1$} & \multicolumn{1}{c}{$k = 2$} \\
  \midrule
  \texttt{cub06}  & \b{1.00} & \b{1.00} \\
  \texttt{cub08}  & \b{1.00} & \b{1.00} \\
  \texttt{cub10}  & \b{1.00} & \b{1.00} \\
  \texttt{cub12}  & \b{1.00} & \b{1.00} \\
  \texttt{cub14}  & \b{1.00} & \b{1.00} \\
  \bottomrule
\end{tabular}
\endgroup
\end{table}

\begin{table}[tbp]
  \centering
  \caption{Success rate~($\uparrow$) for distinguishing pairs of \emph{minimal
  Cayley graphs} when using the alpha complex filtration described 
  at~\Cref{sec:Additional results on alpha complexes} and persistence diagrams up to 
  varying dimensions~(denoted by~$k$). For the results using the other filtrations,
  see~\cref{tab:Minimal Cayley graphs}.}
  \label{tab:Minimal Cayley graphs alpha complexes}
  \smallskip
  \sisetup{
    table-format    = 1.2,
    round-mode      = places,
    round-precision = 2,
    detect-all      = true,
    detect-weight   = true,
  }
  \begingroup
  \let\b=\bfseries
  \let\e=\itshape
  \setlength{\tabcolsep}{2pt}
\begin{tabular}{l S S}
  \toprule
  Data & \multicolumn{1}{c}{$k = 1$} & \multicolumn{1}{c}{$k = 2$} \\
  \midrule
  \texttt{cay12}  & \b{1.00} & \b{1.00} \\
  \texttt{cay16}  & \b{1.00} & \b{1.00} \\
  \texttt{cay20}  & \b{1.00} & \b{1.00} \\
  \texttt{cay24}  & \b{1.00} & \b{1.00} \\
  \texttt{cay32}  & \b{1.00} & \b{1.00} \\
  \texttt{cay36}  & \b{1.00} & \b{1.00} \\
  \texttt{cay60}  & \b{1.00} & \b{1.00} \\
  \texttt{cay63}  & \b{1.00} & \b{1.00} \\
  \bottomrule
\end{tabular}
\endgroup
\end{table}

\begin{table}[tbp]
  \centering
  \caption{Success rate~($\uparrow$) for distinguishing pairs of
  \emph{strongly-regular graphs} when using the alpha complex filtration described 
  at~\Cref{sec:Additional results on alpha complexes} and persistence diagrams up to 
  varying dimensions~(denoted by~$k$). \kWL[2] cannot distinguish between
  any of these pairs. For the results using the other filtrations,
  see~\cref{tab:Strongly-regular graphs}.}
  \label{tab:Strongly-regular graphs alpha complexes}
  \smallskip
  \sisetup{
    table-format    = 1.2,
    round-mode      = places,
    round-precision = 2,
    detect-all      = true,
    detect-weight   = true,
  }
  \begingroup
  \let\b=\bfseries
  \let\e=\itshape
  \setlength{\tabcolsep}{2pt}
\begin{tabular}{l S S S}
  \toprule
  Data & \multicolumn{1}{c}{$k = 1$} & \multicolumn{1}{c}{$k = 2$} & \multicolumn{1}{c}{$k = 3$} \\
  \midrule
  \texttt{16622}  & \b{1.00} & \b{1.00} & \b{1.00} \\
  \texttt{251256} & \b{1.00} & \b{1.00} & \b{1.00} \\
  \texttt{261034} & \b{1.00} & \b{1.00} & \b{1.00} \\
  \texttt{281264} & \b{1.00} & \b{1.00} & \b{1.00} \\
  \texttt{291467} & \b{1.00} & \b{1.00} & \b{1.00} \\
  \texttt{351668} & \b{1.00} & \b{1.00} & \b{1.00} \\
  \texttt{351899} & \b{1.00} & \b{1.00} & \b{1.00} \\
  \texttt{361446} & \b{1.00} & \b{1.00} & \b{1.00} \\
  \texttt{401224} & \b{1.00} & \b{1.00} & \b{1.00} \\
  \bottomrule
\end{tabular}
\endgroup
\end{table}

\begin{table}[tbp]
  \centering
  \caption{Accuracy~($\uparrow$) when predicting the properties of graphs in the
  \texttt{ogbg-molhiv} molecular graph data set~\autocite{hu20a} using the alpha complex filtration described 
  at~\Cref{sec:Additional results on alpha complexes} and persistence diagrams up to 
  varying dimensions~(denoted by~$k$). For the results using the other filtrations,
  see~\cref{tab:Radii prediction experiments}.
  }
  \label{tab:Radii prediction experiments alpha complexes}
  \smallskip
  \sisetup{
    table-format    = 1.2,
    round-mode      = places,
    round-precision = 2,
    detect-all      = true,
    detect-weight   = true,
  }
  \begingroup
  \let\b=\bfseries
  \let\e=\itshape
  \setlength{\tabcolsep}{2pt}
\begin{tabular}{l S S}
  \toprule
  Data & \multicolumn{1}{c}{$k = 1$} & \multicolumn{1}{c}{$k = 2$} \\
  \midrule
  Diameter & 0.00 & \b{0.01} \\
  Girth     & \b{0.41} & 0.40 \\
  Radius   & 0.00 & \b{0.01} \\
  \bottomrule
\end{tabular}
\endgroup
\end{table}

\section{Implementation and Hardware details}
\label{sec:Implementation and Hardware}

Our implementation is based on Python~3. We plan on releasing the code
under a BSD-3-Clause license. For review purposes, the code has been
attached to the supplementary materials. Please use \verb|pip install -r requirements.txt| on the root folder of the code to set up the environment. 
The experiments were executed on a server with an 
AMD EPYC 7452 (128) @ 2.350GHz CPU, 503GiB of RAM memory, 
no GPU acceleration, and
Ubuntu 22.04.4 LTS with the 6.5.0-28-generic Linux kernel. 
Each experiment was executed on a single process. Not completed experiments 
in the main text were due to process termination by the operating system
with the previous constraints.

\begin{compactitem}
  \item To perform the BREC experiments, use the script \verb|BREC_experiments.py|.
  \item To perform the classification experiments, use the script \verb|classification_tasks.py|.
  \item To perform the three non-BREC expressivity experiments (cubic, regular, and Cayley graphs), use the script \verb|simple_isomorphism_experiments.py|.
  \item To perform the graph properties prediction experiments, use the code \verb|property_prediction_experiments.py|.
  \item To perform the graph properties prediction experiments for the Watts--Strogatz and Erdős--Rényi random graphs, use the code \verb|predict_diameter_random_graphs.py|.
\end{compactitem}

The easiest way to perform the experiments in all the cases is using the flag \texttt{-a}, that executes all the experiments without the need of specifying specific parameters about the datasets and filtrations. Once the experiments are performed, the tables from the paper can be generated using the flag \texttt{-t} on the same scripts.

\section{Licenses}
\label{sec:Licenses}

We do not redistribute any existing data sets, but briefly mention their
licenses here. Expressivity experiments on known graphs make use of
an existing database~\autocite{Coolsaet23a}, which does not directly
specify any licensing requirements but asks for a citation.
The BREC data set and the \texttt{ogbg-molhiv} data set are distributed
under a MIT license.
All other graph-classification data sets, i.e.\  \texttt{IMDB-Binary},
\texttt{PROTEINS}, \texttt{MUTAG}, and \texttt{NCI1}, do not specify
a license. They are distributed as part of the `TUDatasets' repository, 
and can be accessed via \url{https://chrsmrrs.github.io/datasets/}.
We make our code available under a 3-Clause BSD License.

\end{document}